\documentclass{article} % For LaTeX2e
\usepackage{iclr2024_conference,times}

\usepackage{amsmath,amsfonts,bm}

\def\eqref#1{equation~\ref{#1}}
\def\1{\bm{1}}

\DeclareMathAlphabet{\mathsfit}{\encodingdefault}{\sfdefault}{m}{sl}
\SetMathAlphabet{\mathsfit}{bold}{\encodingdefault}{\sfdefault}{bx}{n}

\newcommand{\E}{\mathbb{E}}

\DeclareMathOperator*{\argmax}{arg\,max}
\DeclareMathOperator*{\argmin}{arg\,min}

\usepackage{authblk}

\usepackage{hyperref}
\usepackage{url}

\usepackage{enumitem}

\usepackage{microtype}
\usepackage{graphicx}
\usepackage{subcaption}
\usepackage{hyperref}
\usepackage{caption}
\usepackage{wrapfig}

\usepackage[utf8]{inputenc} % allow utf-8 input
\usepackage[T1]{fontenc}    % use 8-bit T1 fonts
\usepackage{hyperref}       % hyperlinks
\usepackage{url}            % simple URL typesetting
\usepackage{booktabs}       % professional-quality tables
\usepackage{amsfonts}       % blackboard math symbols
\usepackage{nicefrac}       % compact symbols for 1/2, etc.
\usepackage{microtype}      % microtypography
\usepackage{xcolor}         % colors

\usepackage{amsmath}
\usepackage{amssymb}
\usepackage{mathtools}
\usepackage{amsthm}
\usepackage{bm}
\usepackage{algorithm}
\usepackage{algorithmic}
\usepackage{color}

\theoremstyle{plain}
\newtheorem{theorem}{Theorem}[section]
\newtheorem{proposition}[theorem]{Proposition}

\theoremstyle{definition}

\theoremstyle{remark}

\usepackage[textsize=tiny]{todonotes}

\usepackage{multirow}

\newcommand{\x}{\mathbf{x}}
\newcommand{\y}{\mathbf{y}}
\newcommand{\z}{\mathbf{z}}
\newcommand{\bb}{\mathbf{b}}
\newcommand{\X}{\mathbf{X}}
\newcommand{\Y}{\mathbf{Y}}
\newcommand{\B}{\mathbf{B}}
\newcommand{\bL}{\mathbf{L}}
\newcommand{\e}{\mathbf{e}}

\newcommand{\w}{\mathbf{w}}

\title{Foundation Model-oriented Robustness: Robust Image Model Evaluation with Pretrained Models}

\makeatletter
\newcommand\email[2][]%
   {\newaffiltrue\let\AB@blk@and\AB@pand
      \if\relax#1\relax\def\AB@note{\AB@thenote}\else\def\AB@note{\relax}%
        \setcounter{Maxaffil}{0}\fi
      \begingroup
        \let\protect\@unexpandable@protect
        \def\thanks{\protect\thanks}\def\footnote{\protect\footnote}%
        \@temptokena=\expandafter{\AB@authors}%
        {\def\\{\protect\\\protect\Affilfont}\xdef\AB@temp{#2}}%
         \xdef\AB@authors{\the\@temptokena\AB@las\AB@au@str
         \protect\\[\affilsep]\protect\Affilfont\AB@temp}%
         \gdef\AB@las{}\gdef\AB@au@str{}%
        {\def\\{, \ignorespaces}\xdef\AB@temp{#2}}%
        \@temptokena=\expandafter{\AB@affillist}%
        \xdef\AB@affillist{\the\@temptokena \AB@affilsep
          \AB@affilnote{}\protect\Affilfont\AB@temp}%
      \endgroup
       \let\AB@affilsep\AB@affilsepx
}
\makeatother

\author[1]{Peiyan Zhang}
\author[2]{Haoyang Liu}
\author[$3^{*}$]{Chaozhuo Li}
\author[3]{Xing Xie}
\author[1]{Sunghun Kim}
\author[2\thanks{Co-corresponding authors}]{Haohan Wang}
\affil[1]{Hong Kong University of Science and Technology}
\email{\url{{pzhangao,hunkim}@cse.ust.hk}}
\affil[2]{University of Illinois at Urbana-Champaign}
\email{\url{{hl57, haohanw}@illinois.edu}}
\affil[3]{Microsoft Research Asia}
\email{\url{{cli,xingx}@microsoft.com}}

\iclrfinalcopy % Uncomment for camera-ready version, but NOT for submission.
\begin{document}

\maketitle

\begin{abstract}
Machine learning has demonstrated remarkable performance over finite datasets, 
yet whether the scores over the fixed benchmarks can sufficiently indicate
the model's performance in the real world 
is still in discussion. In reality, an ideal robust model will probably  
behave similarly to the oracle (\textit{e.g.}, the human users), 
thus a good evaluation protocol is probably to evaluate the models' behaviors 
in comparison to the oracle. 
In this paper, we introduce a new robustness measurement that directly measures the image classification model’s performance compared with a surrogate oracle (\textit{i.e.,} a zoo of foundation models). 
Besides, we design a simple method that can accomplish the evaluation beyond the scope of the benchmarks.
Our method extends the image datasets 
with new samples that are sufficiently perturbed 
to be distinct from the ones in the original sets, 
but are still bounded within the same image-label structure 
the original test image represents, 
constrained by a zoo of foundation models  
pretrained with a large amount of samples. 
As a result, 
our new method will offer us a new way 
to evaluate the models' robustness performance,
free of limitations of fixed benchmarks or constrained perturbations, 
although scoped by the power of the oracle.  
In addition to the evaluation results, 
we also leverage our generated data
to understand the behaviors of the model 
and our new evaluation strategies.

\end{abstract}

\section{Introduction}
\label{sec:intro}
Machine learning has achieved remarkable performance over various benchmarks. For example, the recent successes of various neural network architectures~\citep{he2016deep,touvron2021training} has shown strong numerical evidence 
that the prediction accuracy over specific tasks
can reach the position of the leaderboard 
as high as a human,
% \citep{he2015delving,nangia2019human}, 
suggesting different application scenarios of these methods. 
However, these methods deployed in the real world often 
underdeliver its promises made through the benchmark datasets \citep{Edwards2019,d2020underspecification}, 
usually due to the fact that 
these benchmark datasets, 
typically \textit{i.i.d},
cannot sufficiently 
represent the diversity of the samples 
a model will encounter after being deployed in practice~\citep{recht2019imagenet, wu2023discover}.

Fortunately, multiple lines of study have aimed to embrace this challenge, 
and most of these works are proposing 
to further diversify the datasets used at the evaluation time. 
We notice these works mostly fall into two main categories:
(1) the works that study the performance over testing datasets
generated by predefined perturbation over the original \textit{i.i.d} datasets,
such as adversarial robustness \citep{szegedy2013intriguing,goodfellow2015explaining} or robustness against certain noises \citep{geirhos2018imagenettrained,wang2020high};
and (2) the works that study the performance over testing datasets 
that are 
collected anew with a procedure/distribution different from the one for training sets, 
% \begin{wrapfigure}{R}{0.5\textwidth} 
% \includegraphics[width=0.5\textwidth]{imgs/opening.pdf} 
% \caption{The main structure of our system to generate test images with foundation models and examples of the generated images with their effectiveness in evaluation of model's robustness.}
% \label{fig:main}
% \noindent 
% \vspace{-15pt}
% \end{wrapfigure}
such as domain adaptation \citep{ben2007analysis,ben2010theory} and domain generalization \citep{muandet2013domain}.

\begin{figure*}[t]
\centering
\includegraphics[width=0.9\textwidth]{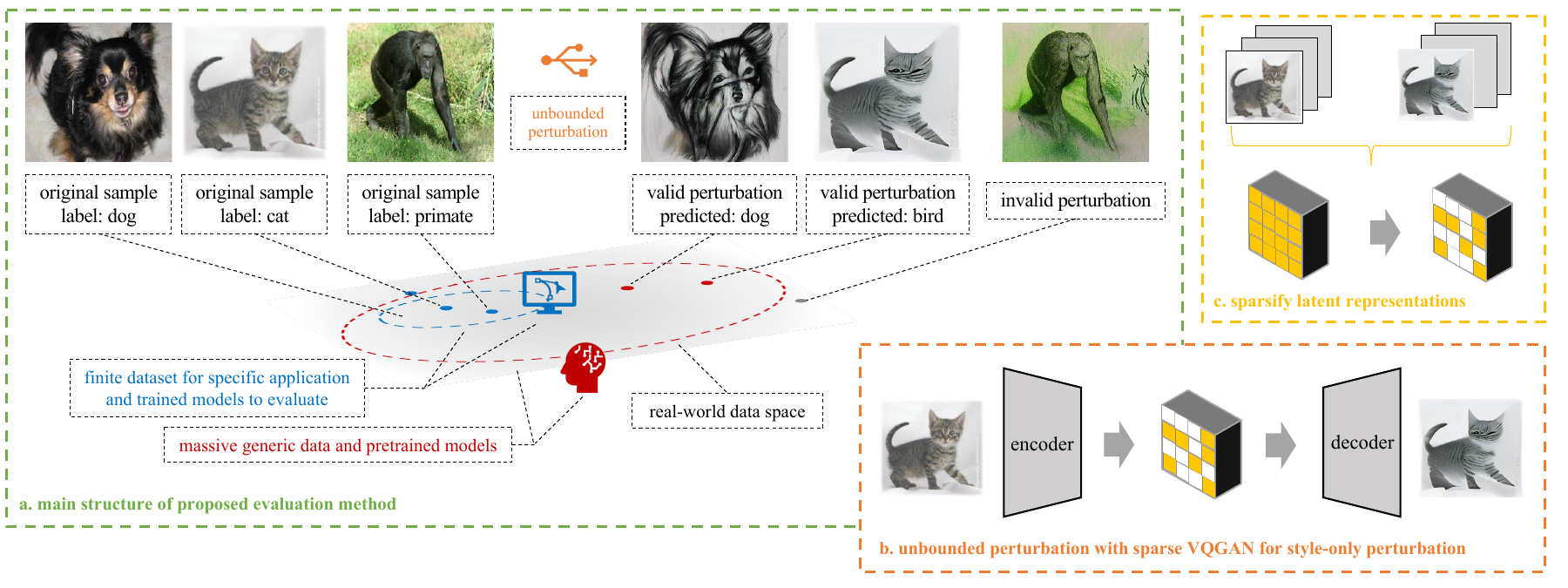} 
\caption{The main structure of our system to generate test images with foundation models and examples of the generated images with their effectiveness in evaluation of model's robustness.}
\label{fig:main}
\noindent 
% \vspace{-0.3cm}
\end{figure*}

% \begin{figure}
%     \centering
%     \includegraphics[width=0.45\textwidth]{opening.pdf}
%     \caption{the main structure of our system to generate test images with surrogate oracle and examples of the generated images with their effectiveness in evaluation of model's robustness.}
%     \label{fig:main}
% \end{figure}

% \begin{figure*}
% \centering
% \includegraphics[width=0.80\textwidth]{imgs/opening.pdf}
% \caption{The main structure of our system to generate test images with foundation models and examples of the generated images with their effectiveness in evaluation of model's robustness.}
% \label{fig:main}
% \end{figure*}

Both of these lines, while pushing the study of robustness evaluation further, 
mostly have their own 
advantages and limitations as a tradeoff on how to guarantee the underlying image-label structure of  
evaluation samples will be the same as the training samples:  
perturbation based evaluations usually 
maintain the image-label structure by predefining the perturbations within a set of operations that will not alter the image semantics, such as 
$\ell$-norm ball constraints \citep{carlini2019evaluating}, or texture \citep{geirhos2018imagenettrained}, frequency-based \citep{wang2020high} perturbations, but are relatively limited in the variety of perturbations allowed.
On the other hand, new-dataset based evaluations maintain the image-label structure by 
soliciting the efforts of human annotators to construct datasets with the same semantics but significantly different styles \citep{hendrycks2021natural,hendrycks2019benchmarking}. However, such new datasets may be costly to collect, and a potential issue is that they are fixed once collected and published to the research community. Ranking methods based on the fixed datasets will eventually lead to the methods overfit on certain datasets~\citep{duda1973pattern,friedman2001elements}.
% selection bias of methods~\citep{duda1973pattern,friedman2001elements}, where the final selection of methods only fits certain datasets rather than being a good reflection of the world. 
While recent efforts have tried to alleviate the selection bias by collecting data from multiple sources ~\citep{gulrajani2020search,koh2021wilds,ye2021ood,wang2022adaptive}, we kindly argue that a dynamic process of generating evaluation datasets will certainly further mitigate this issue.

% More details of these lines and their advantages and limitations 
% and how our proposed evaluation protocol will contrast them will be discussed 
% in the next section. 

% The goal of our paper is specifically not to introduce a new robustness intervention method or image dataset. Instead, our paper is a study of current robustness research to identify the robustness gap between existing models and the oracle. This is particularly important if the ultimate goal is to produce models that function reliably to have oracle-parallel performance.

In this paper, we investigate how to diversify the robustness evaluation datasets to make the evaluation results credible and representative. 
% \py{is this recap appropriate here?} \hwc{I think so, it reads quite good} 
As shown in Figure~\ref{fig:main}, we aim to integrate the advantages of the above two directions by introducing a new protocol to generate evaluation datasets that can automatically perturb the samples to be sufficiently different from existing test samples, while maintaining the underlying unknown image-label structure with respect to a zoo of foundation models. Based on the new evaluation protocol, we introduce a new robustness metric that measures the robustness compared with the foundation model. Moreover, with our proposed evaluation protocol and metric, we make a study of current robust machine learning techniques to identify the robustness 
gap between existing models and the foundation model. 
% This is particularly important if the goal of a research direction is to produce models that function reliably comparable to the foundation model.

% This is particularly important if the goal of a research direction is to produce models that function reliably to have performance comparable to the foundation model.

% In this paper, we aim to integrate the advantages of these two directions
% by introducing a new protocol to generate evaluation datasets 
% that can automatically perturb the samples to be 
% sufficiently different from existing test samples, 
% while maintaining the underlying unknown causal structure.

Therefore, our contributions in this paper are three-fold:
\begin{itemize}[leftmargin=*,noitemsep,topsep=0pt]
    \item We introduce a new robustness metric that measures the robustness gap between models and the foundation model.
    \item We introduce a new evaluation protocol to generate evaluation datasets that can automatically perturb the samples to be sufficiently different from existing test samples, while maintaining the underlying unknown image-label structure.
    \item We leverage our protocol to conduct a systematic study on robustness evaluation, providing insights about deep learning model behavior, opening up future research directions.
    
    % our evaluation metric and protocol to conduct the very first systematic study on robustness evaluation. 
    % % We believe that such a study is timely and significant to the community. 
    % Our analysis brings us the understanding and conjectures of the behavior of the deep learning models, opening up future research directions.

    % \item We introduce a new evaluation protocol that can generate sufficiently perturbed samples from the original samples while maintaining the causal structures by assuming an oracle. 
    % \item We hope our new protocol will serve as a new target for vision machine learning research, since the evaluation datasets will be sufficiently different from training ones, dynamically generated, and maintains the causal structure, and can be used together with most existing vision benchmarks. 
    % \item We leverage our evaluation protocol to further offer a meta-study of current robustness research to identify overarching trends that span multiple evaluation settings.
\end{itemize}
% \hwc{This overall still reads very confident, please rewrite them with your own words, especially the item lists below}

\section{Background}
\label{sec:background}
% In this section, we will discuss the background of our works in two threads: 
% in the first part, we will introduce the preceding efforts 
% in introducing new test protocol to preemptively evaluate model's performances when deployed in the real world and contrast the contribution of our protocol with these preceding efforts;
% in the second part, we will discuss several related works from methodology perspective and 
% highlight our technical contributions. 

% \hwc{shall we re-write this part also to fit the story of OOR?}
% \py{sure.}

% \subsection{The Advantages and Limitations of Current Robustness Evaluation Protocols}
% \vspace{-10pt}
\paragraph{Current Robustness Evaluation Protocols.}
The evaluation of machine learning models in non-\textit{i.i.d} scenario have been studied 
for more than a decade, and one of the pioneers is probably \emph{domain adaptation} \citep{ben2010theory}. 
In domain adaptation, the community trains the model over data from one distribution and tests the model with samples from a different distribution; in \emph{domain generalization} \citep{muandet2013domain}, the community trains the model over data from several related distributions and test the model with samples from yet another distribution. 
% To be more specific, 
% A popular benchmark dataset used in domain generalization study is the PACS dataset \citep{li2017deeper}, which consists of images from seven categories and four domains (photo, art, cartoon, and sketch), and the community studies the empirical performance of models when trained over three of the domains and tested over the remaining one. 
To facilitate the development of cross-domain robust image classification, the community has introduced several benchmarks, 
such as PACS \citep{li2017deeper}, ImageNet-A \citep{hendrycks2021natural}, ImageNet-C \citep{hendrycks2019benchmarking}, ImageNet-Sketch \citep{wang2019learning}, and collective benchmarks integrating multiple datasets such as 
% DomainBed \citep{gulrajani2020search}, 
WILDS \citep{koh2021wilds}, and OOD Bench \citep{ye2021ood}.

% While these benchmarks have significantly pushed the development of robust machine learning, a follow-up question is whether the community should continue to collect more and more benchmark datasets with distribution shifts, considering the human efforts required in the curation of one such datasets. 

While these datasets clearly maintain the underlying image-label structure of the images, 
a potential issue is that these evaluation datasets are fixed once collected. 
Thus, if the community relies on these fixed benchmarks repeatedly to rank methods, 
eventually the selected best method may not be a true reflection of the world, but a model that can fit certain datasets exceptionally well. 
This phenomenon has been discussed by several textbooks \citep{duda1973pattern,friedman2001elements}. 
While recent efforts in evaluating collections of datasets \citep{gulrajani2020search,koh2021wilds,ye2021ood} 
might alleviate the above potential hazards of ``model selection with test set'', 
a dynamic process of generating evaluation datasets will certainly further mitigate this issue. 

On the other hand, one can also test the robustness of models by dynamically perturbing the existing datasets. For example, one can test the model's robustness against rotation \citep{marcos2016learning}, texture \citep{geirhos2018imagenettrained}, frequency-perturbed datasets \citep{wang2020high}, or adversarial attacks (\textit{e.g.}, $\ell_p$-norm constraint perturbations) \citep{szegedy2013intriguing}. 
While these tests do not require additionally collected samples, 
these tests typically limit the perturbations to be relatively well-defined
(\textit{e.g.}, a texture-perturbed cat image still depicts a cat because the shape of the cat is 
preserved during the perturbation). 

While this perturbation test strategy leads to datasets dynamically generated along the evaluation, 
it is usually limited by the variations of the perturbations allowed. 
For example, one may not be able to use some significant distortion of the images 
in case the object depicted may be deformed and the underlying image-label structure of the images is distorted. 
Generally speaking,
% More generally speaking, 
most of the current perturbation-based test protocols are scoped by the tradeoff that a minor perturbation might not introduce enough variations to the existing datasets, 
while a significant perturbation will potentially destroy the underlying image-label structures. 

% \vspace{-10pt}
\paragraph{Assumed Desiderata of Robustness Evaluation Protocol.}
As a reflection of the previous discussion, 
we attempt to offer a summary list of three desired properties 
of the datasets serving as the benchmarks for robustness evaluation: 
\begin{itemize}[leftmargin=*,noitemsep,topsep=0pt]
    \item \textbf{Stableness in Image-label Structure:} 
    the most important property of the evaluation datasets is that the samples must represent the same underlying image-label structure as the training samples. 
    % the most important property of the evaluation datasets is that the samples must represent the same underlying image-label structure as the one in the training samples. 
    % \item \textbf{Diversity in Generated Samples:} for any other non-causal factors of the data, the test samples should cover as many as possible scenarios of the images, such as texture, styles \textit{etc.}
    \item \textbf{Effectiveness in Generated Samples:} the test samples should be effective in exposing defects for tested models. 
    \item \textbf{A Dynamic Generation Process:} to mitigate selection bias of the models over techniques that focus too attentively to the specification of datasets, ideally, the evaluation protocol should consist of a dynamic set of samples, preferably generated with the tested model in consideration.  
\end{itemize}

% \emph{Key Contribution:} 
% To the best of our knowledge, 
% there are no other evaluation protocols of model robustness 
% that can meet the above three properties simultaneously. 
% Thus, 
% we aim to introduce a method that can evaluate model's robustness 
% that fulfill the three above desiderata 
% at the same time. 

% \vspace{-10pt}
\paragraph{Necessity of New Robustness Measurement in Dynamic Evaluation Protocol.}
In previous  experiments, two settings are commonly used: a "standard" test set and a perturbed test set. Previous approaches rank models based on accuracy under perturbed test set~\citep{geirhos2018imagenettrained,hendrycks2021many,orhan2019robustness,xie2020self,zhang2019making} or other metrics such as inception score~\citep{salimans2016improved}, effective robustness~\citep{taori2020measuring} and relative robustness~\citep{taori2020measuring}. While useful for initial assessments, these metrics do not fully capture robustness in dynamic evaluation protocols. Here, comparing two models on different dynamic test sets cannot definitively determine superior model robustness, as differences in performance may result from varying test set difficulties.

% we always have two evaluation settings: the “standard” test set, and the perturbed test set. When comparing the robustness of two models, prior arts would be to rank the models by their accuracy under perturbed test set~\citep{geirhos2018imagenettrained,hendrycks2021many,orhan2019robustness,xie2020self,zhang2019making} or other quantities distinct from accuracy, \textit{e.g.,} inception score~\citep{salimans2016improved}, effective robustness~\citep{taori2020measuring} and relative robustness~\citep{taori2020measuring}.
% These metrics are good starting points for experiments since they are precisely defined and easy to apply to evaluate robustness interventions. In the dynamic evaluation protocols, however, these quantities alone cannot provide a comprehensive measure of robustness, as two models are tested on two different “dynamical” test sets. When one model outperforms the other, we cannot distinguish whether one model is actually better than the other, or if the test set happened to be easier. 

The main challenge identified is the lack of a consistent robustness metric across test sets. Ideally, a robust model should mirror the behavior of the foundation model (e.g., human users). Therefore, a direct measurement of model robustness relative to the foundation model is preferable over indirect model comparisons.
% The core issue in the preceding example is that we can not find the consistent robustness measurement between two different test sets. In reality, an ideal
% robust model will probably behave similarly to the foundation model (e.g., the human users). Thus, instead of indirectly comparing models' robustness with each other, a measurement that directly measures models' robustness compared with the foundation model is desired.

% In this project, 
% we aim to integrate the merits of the two aforementioned strategies of image classification robustness test, 
% to introduce a test protocol that can introduce significant variations of the existing data while being free from additional human efforts. 
% Fortunately, with the help of the foundation models such as CLIP \citep{radford2021learning}, 
% we can replace the human efforts with simple probes of the perturbed images to the CLIP model. 

\section{Method - Counterfactual Generation with Surrogate Oracle}
\label{sec:method}
% In this section, we will first briefly sketch 
% our heuristic algorithm and then introduce 
% the concrete implementation.
%of our method. 

% \vspace{-10pt}
\subsection{Method Overview}
We use $(\x,\y)$ to denote an image sample and its corresponding label, and use $\theta(\x)$ to denote the model we aim to evaluate, which takes an input of the image and predicts the label. 

\begin{wrapfigure}{R}{0.48\textwidth}
    \begin{minipage}{0.48\textwidth}
% \vspace{-0.78cm}
\begin{algorithm}[H]
\begin{algorithmic}
\small
\STATE {\bfseries Input:} $(\X, \Y)$, $\theta$, $g$, $h$, total number of iterations $\B$
\STATE {\bfseries Output:} generated dataset $(\widehat{\X}, \Y)$
\FOR{each $(\x,\y)$ in $(\X, \Y)$}
\STATE generate $\widehat{\x}_0 = g(\x, \bb_0;\theta)$
\IF{$h(\widehat{\x}_0)=\y$}
\STATE set $\widehat{\x} = \widehat{\x}_0$
\FOR{iteration $\bb_t < \B$}
\STATE generate $\widehat{\x}_t = g(\widehat{\x}_{t-1}, \bb_t;\theta)$
\IF{$h(\widehat{\x}_t)=\y$}
\STATE set $\widehat{\x} = \widehat{\x}_t$
\ELSE 
\STATE set $\widehat{\x} = \widehat{\x}_{t-1}$
\STATE exit FOR loop 
\ENDIF
\ENDFOR
\ELSE
\STATE set $\widehat{\x} = \x$
\ENDIF
\STATE use $(\widehat{\x},\y)$ to construct $(\widehat{\X}, \Y)$
\ENDFOR
 
 \end{algorithmic}
 \caption{Perturbed Image Generation with Foundation Models}
 \label{alg:main}
\end{algorithm}
\end{minipage}
% \vspace{-0.3cm}
\end{wrapfigure}

We use $g(\x, \bb)$ to denote an image generation system, which takes an input of the starting image $\x$ to generate another image $\widehat{\x}$ within the computation budget $\bb$. 
The generation process is performed as an optimization process to maximize a scoring function $\alpha(\widehat{\x}, \z)$ that evaluates the alignment between the generated image and generation goal $\z$ guiding the perturbation process. 
The higher the score is, the better the alignment is. 
Thus, the generation process is formalized as 
\begin{align*}
    \widehat{\x} = \argmax_{\widehat{\x} = g(\x, \bb), \bb < \B}\alpha(g(\x, \bb), \z),
    % \vspace{-10pt}
\end{align*}
where $\B$ denotes the allowed computation budget for one sample.
This budget will constrain the generated image not far from the starting image so that the generated one does not converge to a trivial solution that maximizes the scoring function. 

In addition, 
% as one of our aims is to dynamically generate images conditioning on the model, 
we choose the model classification loss $l(\theta(\widehat{\x}), \y)$ as $\z$.
Therefore, the scoring function essentially maximizes the loss of a given image in the direction of a different class. 

Finally, to maintain the unknown image-label structure
of the images, 
we leverage the power of the pretrained giant models
to scope the generation process:
the generated images 
must be considered within the same class 
by the pretrained model, denoted as $h(\widehat{\x})$, 
which takes in the input of the image and makes a prediction. 

Connecting all the components above, 
the generation process will aim to optimize the following: % objective:
\begin{align*}
    \widehat{\x} = & \argmax_{\widehat{\x} = g(\x, \bb), \bb < \B, \z = l(\theta(\widehat{\x}), \y)}
    \alpha(g(\x, \bb), \z), \\
    & \textnormal{subject to} \quad h(\widehat{\x}) = \y .
    % \vspace{-10pt}
\end{align*}

% \begin{align*}
%     \widehat{\x} = & \argmax_{\widehat{\x} = g(\x, \bb), \bb < \B}\alpha(g(\x, \bb), \z) + \lambda l(\theta(\widehat{\x}), \y), \\
%     & \textnormal{subject to} \quad h(\widehat{\x}) = \y ,
% \end{align*}

% where $\lambda$ is a regularization weight balancing the two losses. 

Our method is generic and agnostic to the choices of the three major components, namely $\theta$, $g$, and $h$. 
For example, the $g$ component can vary from something as simple as basic transformations adding noises or rotating images
to a sophisticated method to transfer the style of the images;
on the other hand, the $h$ component can vary from an approach with high reliability and low efficiency such as actually outsourcing the annotation process to human labors
to the other polarity of simply assuming a large-scale pretrained model can function plausibly as a human. 

In the next part, we will introduce our 
concrete choices of $g$ and $h$ leading to the later empirical results,
which build upon the recent advances of vision research.
  
% \begin{algorithm}[t!]
% \begin{algorithmic}
% \STATE {\bfseries Input:} $(\X, \Y)$, $\theta$, $g$, $h$, total number of iterations $\B$
% \STATE {\bfseries Output:} generated dataset $(\widehat{\X}, \Y)$
% \FOR{each $(\x,\y)$ in $(\X, \Y)$}
% \STATE generate $\widehat{\x}_0 = g(\x, \bb_0)$
% \IF{$h(\widehat{\x}_0)=\y$}
% \STATE set $\widehat{\x} = \widehat{\x}_0$
% \FOR{iteration $\bb_t < \B$}
% \STATE generate $\widehat{\x}_t = g(\widehat{\x}_{t-1}, \bb_t)$
% \IF{$h(\widehat{\x}_t)=\y$}
% \STATE set $\widehat{\x} = \widehat{\x}_t$
% \ELSE 
% \STATE set $\widehat{\x} = \widehat{\x}_{t-1}$
% \STATE exit FOR loop 
% \ENDIF
% \ENDFOR
% \ELSE
% \STATE set $\widehat{\x} = \x$
% \ENDIF
% \STATE use $(\widehat{\x},\y)$ to construct $(\widehat{\X}, \Y)$
% \ENDFOR
%  \end{algorithmic}
%  \caption{Counterfactual Image Generation with Plausible Oracle}
%  \hwc{let's also make this a half-page thing}
%  \label{alg:main}
% \end{algorithm}

\subsection{Engineering Specification}

We use VQGAN~\citep{esser2021taming} as the image generation system $g(\x, \bb)$, 
and the $g(\x, \bb)$ is boosted by the evaluated model $\theta(\x)$ 
serving as the $\alpha(\widehat{\x}, \z)$
to guide the generation process, 
where $\z = l(\theta(\widehat{\x}), \y)$ is the model classification loss on current perturbed images. 

% We use VQGAN \citep{esser2021taming} as the image generation system $g(\x, \bb)$, 
% and the $g(\x, \bb)$ is boosted by a CLIP \citep{radford2021learning} model 
% serving as the $\alpha(\widehat{\x}, \z)$
% to guide the generation process, 
% where $\z$ is a textual fragments to describe the images. 
% We design the text fragment to be 
% \textit{``an image of \{class label\} in \{style choice\} style''}, 
% where the \textit{\{style choice\}} are currently among the choices of ``art'', ``cartoon'', ``realism'', and ``sketch'', 
% and the \textit{\{class label\}} will be replaced by the actual label name of that class. 

The generation is an iterative process guided by the scoring function: at each iteration, the system adds more style-wise transformations to the result of the previous iteration. Therefore, the total number of iterations allowed is denoted as the budget $\B$ (see Section~\ref{sec:ss} for details of finding the best perturbation). In practice, the value of budget $\B$ is set based on the resource concerns. 

To guarantee the image-label structure of images, 
we consider using foundation models, \textit{e.g.,} the CLIP~\citep{radford2021learning} model, to serve as $h$, and
design the text fragment input of CLIP to be \textit{``an image of \{class\}''}. However, given the CLIP model's less-than-perfect zero-shot accuracy on most of the base datasets~\citep{radford2021learning}, there exists a potential risk of introducing a label noise to the generated test set. Therefore, in order to reduce the dependency of robustness evaluation on the robustness of the specific foundation model, we employ a majority voting mechanism across an ensemble of multiple foundation models to validate the correctness of labels assigned to the generated images. In our experiments, we assemble a zoo of foundation models, including the CLIP model, ConvNeXt-T-CvSt~\citep{singh2023revisiting} from the RobustBench Leaderboard~\citep{croce2020robustbench}, and CoCa~\citep{Yu2022CoCaCC} from the robust foundation models leaderboard\footnote{\noindent https://paperswithcode.com/sota/zero-shot-transfer-image-classification-on-4}. The ensemble of these foundation models is drawn from diverse sources and exhibits variability, thus enhancing the credibility of label validation for the generated images through a collective majority vote. Afterwards, we directly optimize VQGAN encoder space which is guided by our scoring function.
We show the algorithm in Algorithm~\ref{alg:main}.

\subsection{Measuring Robustness}
% In this part, we will discuss how to measure robustness under this dynamic generation setting.
% In previous  experiments, we always have two evaluation settings: the “standard” test set, and the perturbed test set. When comparing the robustness of two models, prior arts would be to rank the models by their accuracy under perturbed test set~\citep{geirhos2018imagenettrained,hendrycks2021many,orhan2019robustness,xie2020self,zhang2019making} or other quantities distinct from accuracy, \textit{e.g.,} inception score~\citep{salimans2016improved}, effective robustness~\citep{taori2020measuring} and relative robustness~\citep{taori2020measuring}. In our experiments, however, the aforementioned quantities alone cannot provide a comprehensive measure of robustness at this time, as two models are tested on two different “dynamical” test sets. When one model outperforms the other, we cannot distinguish whether one model is actually better than the other, or if the test set happened to be easier. 

\textbf{Foundation Model-oriented Robustness (FMR).} 
% The core issue in the preceding example is that we can not find the consistent robustness measurement between two different test sets. Consider one of the ultimate goals of robustness research is to produce models that function reliably to have oracle-parallel performance. Therefore, instead of indirectly compare models' robustness with each other, we propose to directly measure models' robustness compared with the oracle.
By design, the image-label structures of perturbed images will be maintained by the foundation model. Thus, a smaller accuracy drop on the perturbed images indicates more similar predictions to foundation models. To precisely define FMR, we introduce perturbed accuracy (PA), 
the accuracy on the perturbed images that our generative model successfully produces.
% the accuracy on the images our generation process successfully produces a counterfactual image. 
As the standard accuracy on clean test set (SA) may influence PA to some extent, to disentangle PA from SA, we normalize PA with SA as FMR:
{
\begin{equation}
    \textnormal{FMR} = \frac{\textnormal{PA}}{\textnormal{SA}} \times 100\%\nonumber
    % \vspace{-0.3cm}
\end{equation}}

In settings where the foundation model is human labors, FMR measures the robustness difference between the evaluated model and human perception. In our experiment setting, FMR measures the robustness difference between models trained on fixed datasets (the tested model) and the models trained on unfiltered, highly varied, and highly noisy data (the zoo of foundation models).

% \subsection{The Necessity of the Foundation Model}
At last, 
we devote a short paragraph 
to kindly remind some readers that, 
despite the alluring idea of designing systems 
that forgo the usages of underlying image-label structure
or foundation model, 
it has been proved or argued multiple times 
that it is impossible to 
create that knowledge with nothing but data, 
in either context of machine learning~\citep{locatello2019challenging,mahajan2019preserving,wang2021toward,zhao2022learning,zhao2023beyond} or causality~\citep{bareinboim2020pearl},~\citep[][Sec. 1.4]{pearl2009causality}. 

% Although we utilize the VQGAN as the generator and CLIP as the oracle, we stress that our approach is not specific to either model. We summarize the merits of our method as follows
% \begin{itemize}[leftmargin=*]
%     \item Diversity in the generated counterfactual images that many other non-causal factors of the data would be covered, \textit{e.g.,} background, texture, shape, and styles. We show the counterfactual image samples in Section~\ref{sec:robust} and Appendix~\ref{sec:moresample}.
%     \item Semantic fidelity between labels and generation that governed by the oracle we use, \textit{i.e.,} CLIP.
%     \item Effectiveness of the generated counterfactual images in evaluation. Our method would dynamically adjust the generation strategy based on different model properties, as will be showed in \ref{sec:robust}.
%     \item Efficiency in that our method requires no additional training beyond the pretrained models, using only a small amount of optimization per inference.
%     \item The value of open development and research. This technique will be  in publicly avaliable. We will develop a dynamic evaluation benchmark to serve the adversarial machine learning community. We hope the open collaboration could be integral to its rapid real-world success.
% \end{itemize}

\section{Experiments - Evaluation and Understanding of Models}
\label{sec:evaluation}
% \vspace{-10pt}
\subsection{Experiment Setup}
\label{sec:setup}

We consider four different scenarios, ranging from the basic benchmark MNIST \citep{lecun1998gradient}, through CIFAR10 \citep{krizhevsky2009learning}, 9-class ImageNet \citep{santurkar2019image}, to full-fledged 1000-class ImageNet \citep{deng2009imagenet}. For ImageNet, we resize all images to $224 \times 224~\textrm{px}$. We also center and re-scale the color values with $\mu_{RGB}=[0.485, 0.456, 0.406]$ and $\sigma=[0.229, 0.224, 0.225]$. The perturbation step size for each iteration is 0.001. The total number of iterations allowed (computation budget $\B$) is 50.  

For each of the experiment, we report a set of three results:
\begin{itemize}[leftmargin=*,noitemsep,topsep=0pt]
    \item Standard Accuracy (SA): the clean test accuracy of the evaluated model.
    \item Perturbed Accuracy (PA): accuracy on the images that our generation process successfully produces a perturbed image. 
    % We consider this as the most important metric for the evaluation of the model's robustness. 
    % \py{CA only makes sense if VR is large enough. We need to make this point clear. }
    % \item Validation Rate (VR): the percentage of images validated by the oracle that maintains the causal structure. 
    % \item Mixed Accuracy (MA): we create a mixed dataset, for which if no corresponding counterfactual images are generated, we replace it with the original ones and report the accuracy on this mixed set. \hwc{I think we want to remove MA?}
    \item Foundation Model-oriented Robustness (FMR): robustness of the model compared with the foundation model.
    % \item As SA may influence the CA and MA to some extent, we also report the CA/SA and MA/SA values to normalize the CA and MA based on SA.
\end{itemize}
% \vspace{-0.3cm}
\subsection{Robustness Evaluation for Standard Vision Models}

\label{sec:standard}

We consider a large range of models~(Appedix~\ref{sec:modellist}) and evaluate pre-trained variants of a LeNet architecture \citep{lecun1998gradient} for the MNIST experiment and ResNet architecture \citep{he2016deep} for the remaining experiments. For ImageNet experiment, we also consider pretrained transformer variants of ViT~\citep{dosovitskiy2020image}, Swin~\citep{liu2021swin}, Twins~\citep{chu2021twins}, Visformer~\citep{chen2021visformer} and DeiT~\citep{touvron2021training} from the \textit{timm} library~\citep{rw2019timm}. We evaluate the recent ConvNeXt~\citep{liu2022convnet} as well. All models are trained on the ILSVRC2012 subset of IN comprised of 1.2 million images in the training and a total of 1000 classes~\citep{deng2009imagenet,russakovsky2015imagenet}.

\begin{wraptable}{R}{0.6\textwidth}
\small 
\footnotesize
\centering
\caption{
The robustness test of standard models. We note 1) there exists a performance gap between standard models and the foundation model, and 2) transformer-variants outperform the vanilla ResNet in terms of FMR.}
\begin{tabular}{ccccc}
\hline
\textbf{Data} & \textbf{Model} & \textbf{SA}& \textbf{PA}  & \textbf{FMR}  \\ \hline
{MNIST} & LeNet & {99.09} & 24.76 & 24.99  \\ \hline
{CIFAR10} & ResNet18 & {95.38} & 49.30 & 51.69   \\\hline
{9-class IN} & ResNet18 & {92.30} & 24.89 & 26.97  \\\hline
\multirow{7}{*}{ImageNet} & ResNet50 & {76.26} & 30.59 & 40.11  \\

& ViT & {82.40} & 39.57 & 48.02 \\
& DeiT & {78.57} & 41.18 & 52.41 \\
& Twins & {80.53} & 46.47 & 57.71 \\
& Visformer & {79.88} & 45.71 & 57.22 \\
& Swin & {81.67} & 54.93 & 67.26 \\
& ConvNeXt & {82.05} & 45.44 & 55.38 \\\hline
\end{tabular}
\label{tab:result:main}
% \vspace{-0.3cm}
\end{wraptable}

% \begin{wraptable}{R}{0.6\textwidth}
% \small 
% \footnotesize
% \centering
% \caption{
% The robustness test of standard models. We note 1) there exists a performance gap between standard models and the foundation model, and 2) transformer-variants outperform the vanilla ResNet in terms of FMR.}
% \begin{tabular}{ccccc}
% \hline
% \textbf{Data} & \textbf{Model} & \textbf{SA}& \textbf{PA}  & \textbf{FMR}  \\ \hline
% {MNIST} & LeNet & {99.09} & 27.78 & 28.04  \\ \hline
% {CIFAR10} & ResNet18 & {95.38} & 52.34 & 54.88   \\\hline
% {9-class IN} & ResNet18 & {92.30} & 27.95 & 30.28  \\\hline
% \multirow{7}{*}{ImageNet} & ResNet50 & {76.26} & 33.15 & 43.47  \\

% & ViT & {82.40} & 41.65 & 50.55 \\
% & DeiT & {78.57} & 43.25 & 55.05 \\
% & Twins & {80.53} & 48.52 & 60.25 \\
% & Visformer & {79.88} & 47.82 & 59.87 \\
% & Swin & {81.67} & 56.95 & 69.73 \\
% & ConvNeXt & {82.05} & 47.68 & 58.11 \\\hline
% \end{tabular}
% \label{tab:result:main}
% % \vspace{-0.3cm}
% \end{wraptable}

We report our results in Table~\ref{tab:result:main}. As expected, these models can barely maintain their performances when tested on data from different distributions, as shown by many previous works \citep[\textit{e.g.,}][]{geirhos2018imagenettrained,hendrycks2019benchmarking}. 

Interestingly, on ImageNet, though both transformer-variants models and vanilla CNN-architecture model, \textit{i.e.,} ResNet, attain similar clean image accuracy, transformer-variants substantially outperform ResNet50 in terms of FMR under our dynamic evaluation protocol. We conjecture such a performance gap partly originated from the differences in training setups; more specifically, it may be resulted by the fact transformer-variants by default use strong data augmentation strategies while ResNet50 use none of them. The augmentation strategies (\textit{e.g.,} Mixup~\citep{zhang2017mixup}, Cutmix~\citep{yun2019cutmix} and Random Erasing~\citep{zhong2020random}, \textit{etc.}) already naively introduce OOD samples during training, therefore are potentially helpful for securing model robustness towards data shifts. When equiping with the similar data augmentation strategies, CNN-architecture model, \textit{i.e.,} ConvNext, has achieved comparable performance in terms of FMR. This hypothesis has also been verified in recent works~\citep{bai2021transformers,wang2022can}. We will offer more discussions on the robustness enhancing methods in Section~\ref{sec:robust}. 

Surprisingly, we notice a large difference within the transformer family in the proposed FMR metric. Despite Swin Transformer's suboptimal accuracy on the standard dataset, it achieves the best FMR. One possible reason for this phenomenon may due to their internal architecture that are related to the self-attention mechanism. Therefore, we conduct in-depth analysis on the effects of head numbers. The results reveal that increased head numbers enhance expressivity and robustness, albeit at the expense of clean accuracy. More details can be found in Appendix~\ref{sec:multi-head}.

Besides comparing performance between different standard models, FMR brings us the chance to directly compare models with the foundation model. Across all of our experiments, the FMR shows the significant gap between models and the foundation model, which is trained on the unfiltered and highly varied data, seemingly suggesting that training with a more diverse dataset would help with robustness. This overarching trend has also been identified in~\citep{taori2020measuring}. However, quantifying when and why training with more data helps is still an interesting open question.

\subsection{Robustness Evaluation for Robust Vision Models}
\label{sec:robust}
% Recently, some techniques have been introduced to cope with corruptions or style shifts. For example, by adapting the batch normalization statistics with a limited number of samples~\citep{schneider2020improving}, the performance on stylized images (or corrupted images) can be significantly increased. 
% Additionally, some more sophisticated techniques, \textit{e.g.,} DAT~\citep{mao2022enhance}, have also been widely employed by the community. 

Recently, some techniques have been introduced to cope with corruptions or style shifts. To investigate whether those OOD robust models can maintain the performance in our dynamic evaluation, we evaluate the pretrained ResNet50 models combining with the five leading methods from the ImageNet-C leaderboard, namely, Stylized ImageNet training (SIN;~\citep{geirhos2018imagenettrained}), adversarial
noise training (ANT;~\citep{rusak2020increasing}), a combination of ANT and SIN (ANT+SIN;~\citep{rusak2020increasing}), optimized data augmentation
using Augmix (AugMix;~\citep{hendrycks2019augmix}), DeepAugment (DeepAug;~\citep{hendrycks2021many}), a combination of Augmix and DeepAug (DeepAug+AM;~\citep{hendrycks2021many}) and Discrete Adversarial Training (DAT;~\citep{mao2022enhance}).

% \begin{table}[]
\begin{wraptable}{R}{0.6\textwidth}
\footnotesize
\centering 
\caption{The robustness test of generated countefactual images for OOD robust models. $\rm SA^{*}$ 
represents the model's top-1 accuracy on ImageNet-C dataset. We note that applying DAT yields the best FMR under our dynamic evaluation protocol.}
\begin{tabular}{ccccc}
\hline
\textbf{Model} & \textbf{SA} & $\rm \mathbf{SA^{*}}$ & \textbf{PA}  & \textbf{FMR}  \\ \hline
ResNet50 & {76.26} & 39.20 & 30.59 & 40.11 \\ \hline
ANT & 76.26 & 50.41 & 30.61 & 40.14  \\
SIN & 76.24 & 45.19 & 30.55 & 40.07  \\
ANT+SIN & 76.26 & 52.60 & 31.09 & 40.77  \\
DeepAug & 76.26 & 52.60 & 33.24 & 43.59  \\
Augmix & 76.73 & 48.31 & 38.89 & 50.68  \\
DeepAug+AM & 76.68 & 58.10 & 42.60 & 55.56  \\ 
DAT & 76.57 & 68.00 & 52.57 & 68.66  \\ 
\hline
\end{tabular}
\label{tab:result:defense}
% \vspace{-0.3cm}
% \end{table}
\end{wraptable}

The results are displayed in Table~\ref{tab:result:defense}. Surprisingly, we find that some common corruption robust models, \textit{i.e.,} SIN, ANT, ANT+SIN, fail to maintain their power under our dynamic evaluation protocol. We take the SIN method as an example. The FMR of SIN method is $40.07$, which is even lower than that of a vanilla ResNet50. 
% The MA of SIN method is $45.30$, which also fails to outperform the vanilla ResNet50 model. \hwc{remember to take this sentence out when we take the MA out.}
These methods are well fitted in the benchmark ImageNet-C, verifying the weakness of relying on fixed benchmarks to rank methods. The selected best method may not be a true reflection of the real world, but a model well fits certain datasets, which in turn proves the necessity of our dynamic evaluation protocol.

DeepAug, Augmix, DeepAug+AM perform better than SIN and ANT methods in terms of FMR as they dynamically perturb the datasets, alleviating the hazards of ``model selection with test set'' to some extent. DAT outperform others in terms of FMR, which validates the effectiveness of perturbation in the meaningful symbolic space rather than the continuous pixel space.
However, their performance is limited by the variations of the perturbations allowed, resulting in marginal improvements compared with ResNet50.

\begin{figure*}[t]
    \centering
    
    % \captionsetup{font={small}}
        \begin{subfigure}{0.11\textwidth}
            \centering
            \includegraphics[width=\textwidth]{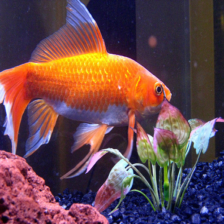} 
            \caption{gold fish} 
            
        \end{subfigure}
        \begin{subfigure}{0.11\textwidth}
            \centering
            \includegraphics[width=\textwidth]{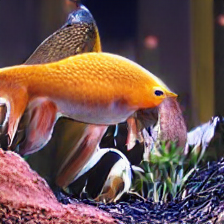} 
            \caption{snoek} 
            \label{fig:sin}
            
        \end{subfigure}
        \begin{subfigure}{0.11\textwidth}
            \centering
            \includegraphics[width=\textwidth]{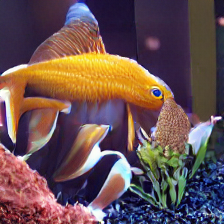} 
            \caption{anemone} 
        \end{subfigure}
        \begin{subfigure}{0.11\textwidth}
            \centering
            \includegraphics[width=\textwidth]{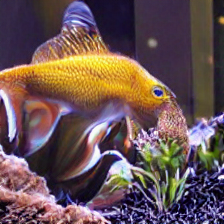} 
            \caption{snoek} 
            \label{fig:antsin}
        \end{subfigure}
        \begin{subfigure}{0.11\textwidth}
            \centering
            \includegraphics[width=\textwidth]{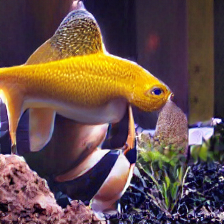} 
            \caption{coralreef} 
        \end{subfigure}
        \begin{subfigure}{0.11\textwidth}
            \centering
            \includegraphics[width=\textwidth]{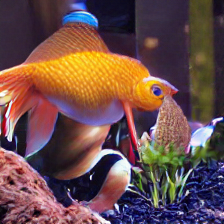} 
            \caption{hen} 
            \label{fig:deepaug}
        \end{subfigure}
        \begin{subfigure}{0.11\textwidth}
            \centering
            \includegraphics[width=\textwidth]{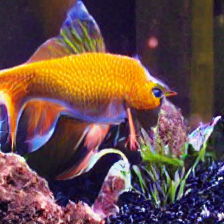} 
            \caption{snoek}
            \label{fig:deepaug+am}
        \end{subfigure}
        \begin{subfigure}{0.11\textwidth}
            \centering
            \includegraphics[width=\textwidth]{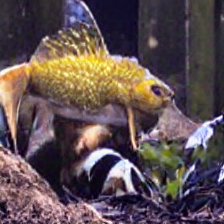} 
            \caption{badger}
            \label{fig:deepaug+am}
        \end{subfigure}
        \\
        \begin{subfigure}{0.11\textwidth}
            \centering
            \includegraphics[width=\textwidth]{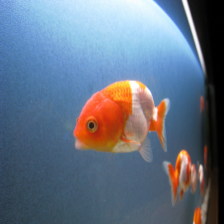} 
            \caption{gold fish} 
        \end{subfigure}
        \begin{subfigure}{0.11\textwidth}
            \centering
            \includegraphics[width=\textwidth]{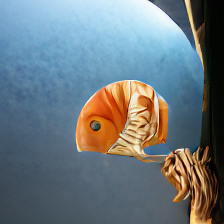} 
            \caption{nautilus} 
            \label{fig:2sin}
        \end{subfigure}
        \begin{subfigure}{0.11\textwidth}
            \centering
            \includegraphics[width=\textwidth]{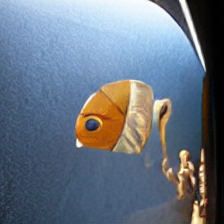}
            \caption{lamp} 
        \end{subfigure}
        \begin{subfigure}{0.11\textwidth}
            \centering
            \includegraphics[width=\textwidth]{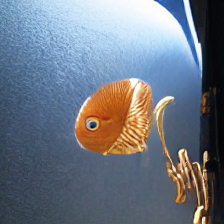} 
            \caption{nautilus} 
            \label{fig:2antsin}
        \end{subfigure}
        \begin{subfigure}{0.11\textwidth}
            \centering
            \includegraphics[width=\textwidth]{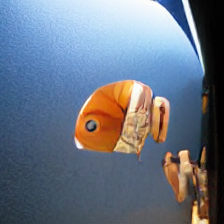} 
            \caption{lamp} 
        \end{subfigure}
        \begin{subfigure}{0.11\textwidth}
            \centering
            \includegraphics[width=\textwidth]{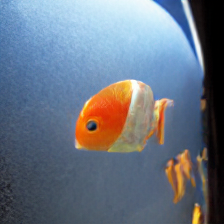} 
            \caption{anemone} 
        \end{subfigure}
        \begin{subfigure}{0.11\textwidth}
            \centering
            \includegraphics[width=\textwidth]{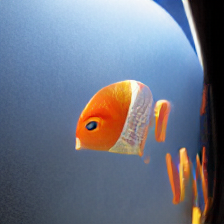}
            \caption{coralreef} 
        \end{subfigure}
        \begin{subfigure}{0.11\textwidth}
            \centering
            \includegraphics[width=\textwidth]{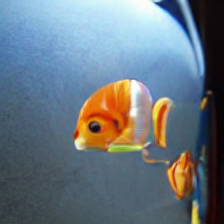}
            \caption{milk can} 
        \end{subfigure}
\caption{Visualization of the images generated by our system in evaluating the common corruption robust model, with the original image shown (left image of each row). The caption for each image is either the original label or the predicted label by the corresponding model. The evaluated models are SIN, ANT, ANT+SIN, Augmix, DeepAug, DeepAug+AM and DAT from left to right.}
\label{fig:style-robust}
% \vspace{-0.3cm}
\end{figure*}

Besides, we visualize the perturbed images generated
according to the evaluated style-shift robust models in Figure~\ref{fig:style-robust}. More results and the discussion on the realism of the generated images are shown in Appendix~\ref{sec:moresample} and~\ref{sec:realism}. We have the following observations:

\textbf{Preservation of Local Textual Details.} 
Recent studies highlight that CNNs often prioritize object textures over shapes for learning~\citep{gatys2015texture,ballester2016performance,gatys2017texture,geirhos2018imagenettrained,wang2020high}. Our perturbed images retain misleading textures, making evaluation more challenging as textures become a nuisance rather than predictive. For instance, in Figure~\ref{fig:deepaug}, we generated images with skin textures resembling chicken skin, which misleads the ResNet with DeepAug method.

% A number of recent empirical findings point to an important role of object textures
% for CNN, where object textures are more important than global object shapes for CNN model to learn ~\citep{gatys2015texture,ballester2016performance,gatys2017texture,geirhos2018imagenettrained,wang2020high}. 
% We notice that our generated perturbed images may preserve false local textual details, leading to a harder evaluation task since textures are no longer predictive, but instead a nuisance factor (as desired). For example, in Figure~\ref{fig:deepaug}, we produce a skin texture similar to chicken skin, and ResNet with DeepAug method is misled by this corruption. 
% % This is possible in common corruptions when random blurs causes the pixel-level errors. However, 

\textbf{Generalization to Shape Perturbations.} Our attack dynamically adjusts intensity using the model's gradient, affecting both texture and shape while preserving image-label structures. This results in successful model attacks with shape-perturbed images, as demonstrated in the SIN (Figure~\ref{fig:sin} and Figure~\ref{fig:2sin}) and ANT+SIN (Figure~\ref{fig:antsin} and Figure~\ref{fig:2antsin}) examples.

% Moreover, since our attack intensity could be dynamically altered based on the model's gradient while still maintaining the image-label structures, the perturbation would be sufficient that not just limited to object textures, but even a certain degree of shape perturbation. We observe that the perturbed image generated for SIN (Figure~\ref{fig:sin} and Figure~\ref{fig:2sin}) and ANT+SIN (Figure~\ref{fig:antsin} and Figure~\ref{fig:2antsin}) are shape-perturbed and successfully attack the models. 

% Moreover, since our attack intensity could be dynamically altered based on the model's gradient while still maintaining the image-label structures, the perturbation we produce would be sufficient that not just limited to object textures, but even a certain degree of shape perturbation. As it is acknowledged that networks with a higher shape bias are inherently more robust to many different
% image distortions and reach higher performance on classification and classification tasks, we observe that the perturbed image generated for SIN (Figure~\ref{fig:sin} and Figure~\ref{fig:2sin}) and ANT+SIN (Figure~\ref{fig:antsin} and Figure~\ref{fig:2antsin}) are shape-perturbed and successfully attack the models. 

\textbf{Recognition of Model Properties.} By integrating various methods, we generate more complex perturbed images, such as those combining DeepAug's chicken-like head with Augmix's skin patterns (Figure~\ref{fig:deepaug+am}), demonstrating our method's ability to adapt to model properties for challenging evaluations. This shows our protocol dynamically tailors attacks to model specifics, producing perturbed images that reveal weaknesses beyond static benchmarks, \textit{i.e.,} ImageNet-C.

\begin{table*}[t]
%\linespread{1.2}
% \setlength{\abovecaptionskip}{0cm} 
% \setlength{\belowcaptionskip}{-0.2cm}
\Large
\centering
  \caption{Study of different image generator choices on ImageNet dataset. The numbers of PA and FMR are reported. The results are consistent under different image generator configurations.}
  \footnotesize
  \label{tab:ablation}
  \resizebox{.8\linewidth}{!}{
  % \begin{tabular}{c|cc|cc|cc|cc|cc}
  %   \toprule
  %   \multirow{2}{*}{Model} & \multicolumn{2}{c}{ADM} & \multicolumn{2}{c}{Improved DDPM} & \multicolumn{2}{c}{Efficient-VDVAE} & \multicolumn{2}{c}{StyleGAN-XL}
  %   & \multicolumn{2}{c}{VQGAN}\\
  %              ~&VR & FMR & VR & FMR & VR & FMR & VR & FMR & VR &  FMR \\
  %   \midrule
  %   ResNet50   & 32.57 & 43.84 & 32.57 & 42.63 & 32.57 & 41.14 & 32.57 & 42.93 &  32.57 & 43.47  \\
  %   ANT & 32.57 & 43.21 & 32.58 & 44.08 & 32.57 & 42.39 & 32.58 & 43.29 &  32.57 & 42.92 \\
    
  %   SIN & 32.57 & 43.58 & 32.57 & 43.43 & 32.57 & 42.32 & 32.58 & 42.58 &  32.57 & 42.90  \\
  %   ANT+SIN & 32.57 & 43.92 & 32.57 & 45.26 & 32.58 & 44.20 & 32.57 & 44.57 & 32.58 & 43.52  \\

  %   DeepAug & 32.57 & 45.04 & 32.57 & 46.47 & 32.57 & 45.75 & 32.57 & 46.57  & 32.57 & 46.33  \\
  %   Augmix & 32.58 & 52.77 & 32.57 & 53.69 & 32.58 & 53.42 & 32.57 & 52.57 & 32.57 & 53.36 \\
  %   DeepAug+AM & 32.58 & 57.98 & 32.57 & 57.65 & 32.57 & 55.23 & 32.57 & 55.64 & 32.58 & 58.19 \\
  %   \bottomrule
  % \end{tabular}
\begin{tabular}{c|cc|cc|cc|cc|cc}
    \toprule
    \multirow{2}{*}{Model} & \multicolumn{2}{c}{ADM} & \multicolumn{2}{c}{Improved DDPM} & \multicolumn{2}{c}{Efficient-VDVAE} & \multicolumn{2}{c}{StyleGAN-XL}
    & \multicolumn{2}{c}{VQGAN}\\
               ~&PA & FMR & PA & FMR & PA & FMR & PA & FMR & PA &  FMR \\
    \midrule
    ResNet50   & 32.36 & 42.43 & 31.43 & 41.21 & 30.28 & 39.71 & 31.65 & 41.50 &  32.09 & 42.08  \\
    ANT & 31.88 & 41.80 & 32.54 & 42.67 & 31.29 & 41.03 & 31.94 & 41.88 &  31.65 & 41.50 \\
    
    SIN & 32.17 & 42.20 & 32.05 & 42.04 & 31.15 & 40.86 & 31.39 & 41.17 &  31.64 & 41.50  \\
    ANT+SIN & 32.47 & 42.58 & 33.50 & 43.93 & 32.68 & 42.85 & 33.01 & 43.29 & 32.15 & 42.16  \\

    DeepAug & 33.32 & 43.69 & 34.39 & 45.10 & 33.83 & 44.36 & 34.46 & 45.19  & 34.30 & 44.98  \\
    Augmix & 39.47 & 51.44 & 40.16 & 52.34 & 39.95 & 52.07 & 39.30 & 51.22 & 40.01 & 52.14 \\
    DeepAug+AM & 43.43 & 56.64 & 43.17 & 56.30 & 41.32 & 53.89 & 41.71 & 54.39 & 43.65 & 56.92 \\
    \bottomrule
  \end{tabular}
  }
% \vspace{-0.3cm}
\end{table*}

% \vspace{-10pt}
\subsection{Understanding the Properties of Our Evaluation System}
% \hwc{maybe we can put these in appendix if we don't have enough space?} \py{yes. I think the last page can be the ablation and conclusion.} \hwc{Sure, but please also write a short summary here when you move it to appendix}

% \py{We need to clearly point out that the generated images' diversity should not be large in structures or styles, but the details.}
We continue to investigate several properties of the models 
in the next couple sections. To save space, we will mainly present the results on CIFAR10 experiment here and save the details to the appendix:
\begin{itemize}[leftmargin=*,noitemsep,topsep=0pt]
    \item In Appendix~\ref{sec:transfer}, we explore the transferability of the generated images and validate the reliability of the FMR metric. The results of a reasonable transferability suggest that our method of generating images is not model-specificity, and
can be potentially used in a broader scope: 
we can leverage the method to generate a static set of images and set a benchmark to help the development of robustness methods. 
    % \item In Appendix~\ref{sec:fgsm}, we find that initiating with FGSM adversarial examples~\citep{goodfellow2015explaining} barely affect FMR.
    \item In Appendix~\ref{sec:pgd}, we compare the vanilla model to a model trained by PGD~\citep{madry2017towards}. We find that these two models process the data differently. However, their robustness weaknesses are exposed to a similar degree by our test system.
    \item In Appendix~\ref{sec:advt}, we investigate enhancing evaluated robustness by training the model with images generated by our evaluation system. Due to computational constraints, we use a static image set for training, which indeed improves model robustness in our system.
%     In Appendix~\ref{sec:advt}, we explore the possibility of 
% improving the evaluated robustness by 
% augmenting the images with those generated by our evaluation system. Due to the required computational load, 
% we only use a static set of generated images to train the model 
% and the results suggest that static set of images 
% for augmentation indeed brings enhanced robustness to the model in our evaluation system.
    \item We also notice that 
the generated images tend to shift the color of the original images, so we tested the robustness of grayscale models in Appendix~\ref{sec:greyscale}, 
the results suggest removing the color channel will not 
improve robustness performances. 
\end{itemize}

\subsection{Experiments Regarding Method Configuration}
\label{sec:ss}

% \vspace{-0.3cm}

\textbf{Generator Configuration.} We conduct ablation study on the generator choice to agree on the performance ranking in Table~\ref{tab:result:main} and Table~\ref{tab:result:defense}. We consider several image generator architechitures, namely, variational autoencoder (VAE)~\citep{kingma2013auto, rezende2014stochastic} like Efficient-VDVAE~\citep{hazami2022efficient}, diffusion models~\citep{sohl2015deep} like Improved DDPM~\citep{nichol2021improved} and ADM~\citep{dhariwal2021diffusion}, and GAN like StyleGAN-XL~\citep{sauer2022stylegan}. As shown in Table~\ref{tab:ablation}, we find that the conclusion is consistent under different generator choices, which validates the correctness of our conclusions in Section~\ref{sec:standard} and Section~\ref{sec:robust}.

\textbf{Sparse VQGAN.} In resource-constrained scenarios, we enhance efficiency by sparsifying VQGAN, discovering that only 0.69\% of dimensions significantly impact style. By masking the remaining 99.31\%, we create a sparse VQGAN submodel, reducing runtime by 12.7\% on 9-class ImageNet and 28.5\% on ImageNet, making our protocol viable even with limited computing resources. Further details are in the Appendix~\ref{sec:sparsevqgan}.

% For scenarios with limited computing resources, we consider sparsifying the VQGAN to speedup the generation process.
% In experiments of sparse VQGAN, 
% we find that only 0.69\% dimensions are highly correlated to the style. Therefore, we mask the rest 99.31\% dimensions to create a sparse submodel of VQGAN for efficient perturbation. The running time can be reduced by 12.7\% on 9-class ImageNet and 28.5\% on ImageNet, respectively. Details can be found in Appendix~\ref{sec:sparsevqgan}. Therefore, our evaluation protocol is also feasible with limited computing resources. 

\textbf{Step Size.} The optimal step size varies with the computation budget (B). Under limited budgets, a larger step size is necessary but may highlight model limitations, affecting evaluation outcomes. With ample B, a smaller step size can alleviate these issues, proving sufficient for practical applications, as detailed in Appendix~\ref{sec:9classdetail}.

\section{Discussion}
\label{sec:discussion}
% \vspace{-10pt}
\label{sec:discuss}
\subsection{Discussion on the Bias Issues}
\textbf{Selection Bias.} 
In previous sections, we have mentioned that ranking models based on the fixed datasets will potentially lead to the selection bias of methods~\citep{duda1973pattern,friedman2001elements}. While our dynamic evaluation protocol help mitigates this issue, it is inevitable to introduce other biases when we select specific generators and foundation models. Here, we provide more analyses and discussions.

\textbf{Bias towards Generators.}
As our evaluation protocol requires an image generator, the quality or diversity of the generated images may be bounded by the choice of generator. However, Table~\ref{tab:ablation} shows the consistent conclusions made in the paper, which verifies that the proposed method is robust to the choice of generator.

\textbf{Bias towards the Foundation Models.} In Appendix~\ref{sec:9classdetail}, we take the CLIP model as an example and explore the category unbalance issue. We observe its performance is affected by imbalanced online sample distributions, leading to perturbed images of varied difficulty. Fortunately, our model configurations significantly mitigate this issue (see Appendix~\ref{sec:9classdetail}), proving effective for real-world use. Additionally, using an ensemble of foundation models enhances this mitigation.

% We find that CLIP has been influenced by the imbalance sample distributions across the Internet, resulting in perturbed images with varying degrees of difficulty. 
% Fortunately, our devised model configurations exhibit a notable capacity to ameliorate this issue to a significant extent, as elucidated in Appendix~\ref{sec:9classdetail}. This level of mitigation proves to be sufficiently efficacious for real-world applications. Additionally, the employment of an ensemble of multiple foundation models in our methodology provides a further layer of alleviation for this issue.

\textbf{Bias of the Metric.}
As the generated samples are biased to the zoo of foundation models' zero-shot performance, "PA" and "FMR" scores will also be biased. In Appendix~\ref{sec:distesti}, we conduct a theoretical analysis to guarantee the correctness of the proposed method. Our theoretical analysis confirms that while both traditional datasets and foundation model zoos can approximate true distributions, the latter achieves this with less variance. Hence, we advocate that bias towards foundation model zoos is preferable to conventional datasets.

\textbf{Potential Negative Impacts.} Although the bias incurred by the zoo of foundation models is less detrimental than the biases arisen from fixed benchmark datasets, a more detailed discussion on the potential negative impacts is necessary. Therefore, we discuss the potential negative impacts as well as the societal bias of relying on large models in Appendix~\ref{sec:societal_bias}.

\subsection{Discussion on the Effects of foundation model's Zero-shot Performance}

\textbf{Domain Gap Concerns.} Despite the zero-shot strengths of our foundation model zoo, it may underperform in niche areas, \textit{e.g.,} medical imaging, where general knowledge falls short. However, our framework's adaptable design allows for the easy inclusion of domain-specific pre-trained models, providing a versatile solution for a wide range of applications.

\textbf{Zero-shot Adversarial Robustness Concerns.} 

Foundation models like CLIP are vulnerable to adversarial attacks~\citep{mao2022understanding}, potentially undermining evaluation effectiveness if attackers access and manipulate CLIP's weights. In Appendix~\ref{zero-shot-adv}, our study into CLIP's zero-shot adversarial robustness reveals it outperforms XCiT-L12~\citep{debenedetti2022light} in resilience against FGSM attacks, despite susceptibility to classification changes. However, gradient masking and other simple techniques can safeguard CLIP in production, significantly reducing white-box attack risks. For black-box attacks, CLIP demonstrates strong resilience (See Appendix~\ref{zero-shot-adv}). By integrating a robust model ensemble and employing a majority vote for image-label validation, our approach enhances security. Therefore, CLIP, particularly when safeguarded by weight and gradient protection techniques and supported by a robust model ensemble, stands as a strong candidate for the ideal foundation model to preserve image-label integrity currently.

\section{Conclusion}
\label{sec:con}
% \vspace{-10pt}
In this paper, 
we first summarized the common practices 
of model evaluation strategies 
for robust vision machine learning. 
We then discussed three desiderata 
of the robustness evaluation protocol. 
Further, we offered a simple method 
that can fulfill these three desiderata
at the same time, 
serving the purpose of evaluating vision models' robustness 
across generic \textit{i.i.d} benchmarks, 
without requirement on the prior knowledge 
of the underlying image-label structure depicted by the images,
although relying on a zoo of foundation models. 
% \section*{Broader Impact}
% \label{sec:impact}
% \input{secs/impact}

% \section*{Ethics Statement}
% \label{sec:ethics}
% \input{secs/ethicstatement}

% \section*{Reproducibility Statement}
% \label{sec:repro}
% \input{secs/reprostatement}

% % Acknowledgements should only appear in the accepted version.
% \section*{Acknowledgements}

% \textbf{Do not} include acknowledgements in the initial version of
% the paper submitted for blind review.

% If a paper is accepted, the final camera-ready version can (and
% probably should) include acknowledgements. In this case, please
% place such acknowledgements in an unnumbered section at the
% end of the paper. Typically, this will include thanks to reviewers
% who gave useful comments, to colleagues who contributed to the ideas,
% and to funding agencies and corporate sponsors that provided financial
% support.

% In the unusual situation where you want a paper to appear in the
% references without citing it in the main text, use \nocite

\bibliography{iclr2024_conference}

\begin{thebibliography}{96}
\providecommand{\natexlab}[1]{#1}
\providecommand{\url}[1]{\texttt{#1}}
\expandafter\ifx\csname urlstyle\endcsname\relax
  \providecommand{\doi}[1]{doi: #1}\else
  \providecommand{\doi}{doi: \begingroup \urlstyle{rm}\Url}\fi

\bibitem[Bai et~al.(2021)Bai, Mei, Yuille, and Xie]{bai2021transformers}
Yutong Bai, Jieru Mei, Alan~L Yuille, and Cihang Xie.
\newblock Are transformers more robust than cnns?
\newblock \emph{Advances in Neural Information Processing Systems},
  34:\penalty0 26831--26843, 2021.

\bibitem[Ballester \& Araujo(2016)Ballester and
  Araujo]{ballester2016performance}
Pedro Ballester and Ricardo~Matsumura Araujo.
\newblock On the performance of googlenet and alexnet applied to sketches.
\newblock In \emph{Thirtieth AAAI Conference on Artificial Intelligence}, 2016.

\bibitem[Bareinboim et~al.(2020)Bareinboim, Correa, Ibeling, and
  Icard]{bareinboim2020pearl}
Elias Bareinboim, Juan~D Correa, Duligur Ibeling, and Thomas Icard.
\newblock On pearl’s hierarchy and the foundations of causal inference.
\newblock \emph{ACM Special Volume in Honor of Judea Pearl (provisional
  title)}, 2\penalty0 (3):\penalty0 4, 2020.

\bibitem[Ben-David et~al.(2007)Ben-David, Blitzer, Crammer, and
  Pereira]{ben2007analysis}
Shai Ben-David, John Blitzer, Koby Crammer, and Fernando Pereira.
\newblock Analysis of representations for domain adaptation.
\newblock In \emph{Advances in neural information processing systems}, pp.\
  137--144, 2007.

\bibitem[Ben-David et~al.(2010)Ben-David, Blitzer, Crammer, Kulesza, Pereira,
  and Vaughan]{ben2010theory}
Shai Ben-David, John Blitzer, Koby Crammer, Alex Kulesza, Fernando Pereira, and
  Jennifer~Wortman Vaughan.
\newblock A theory of learning from different domains.
\newblock \emph{Machine learning}, 79\penalty0 (1-2):\penalty0 151--175, 2010.

\bibitem[Bose et~al.(2023)Bose, Hebbar, Somandepalli, Zhang, Cui,
  Cole-McLaughlin, Wang, and Narayanan]{bose2023movieclip}
Digbalay Bose, Rajat Hebbar, Krishna Somandepalli, Haoyang Zhang, Yin Cui, Kree
  Cole-McLaughlin, Huisheng Wang, and Shrikanth Narayanan.
\newblock Movieclip: Visual scene recognition in movies.
\newblock In \emph{Proceedings of the IEEE/CVF Winter Conference on
  Applications of Computer Vision}, pp.\  2083--2092, 2023.

\bibitem[Carlini et~al.(2019)Carlini, Athalye, Papernot, Brendel, Rauber,
  Tsipras, Goodfellow, Madry, and Kurakin]{carlini2019evaluating}
Nicholas Carlini, Anish Athalye, Nicolas Papernot, Wieland Brendel, Jonas
  Rauber, Dimitris Tsipras, Ian Goodfellow, Aleksander Madry, and Alexey
  Kurakin.
\newblock On evaluating adversarial robustness.
\newblock \emph{arXiv preprint arXiv:1902.06705}, 2019.

\bibitem[Chen et~al.(2021)Chen, Xie, Niu, Liu, Wei, and
  Tian]{chen2021visformer}
Zhengsu Chen, Lingxi Xie, Jianwei Niu, Xuefeng Liu, Longhui Wei, and Qi~Tian.
\newblock Visformer: The vision-friendly transformer.
\newblock In \emph{Proceedings of the IEEE/CVF International Conference on
  Computer Vision}, pp.\  589--598, 2021.

\bibitem[Chu et~al.(2021)Chu, Tian, Wang, Zhang, Ren, Wei, Xia, and
  Shen]{chu2021twins}
Xiangxiang Chu, Zhi Tian, Yuqing Wang, Bo~Zhang, Haibing Ren, Xiaolin Wei,
  Huaxia Xia, and Chunhua Shen.
\newblock Twins: Revisiting the design of spatial attention in vision
  transformers.
\newblock \emph{Advances in Neural Information Processing Systems}, 34, 2021.

\bibitem[Croce et~al.(2020)Croce, Andriushchenko, Sehwag, Debenedetti,
  Flammarion, Chiang, Mittal, and Hein]{croce2020robustbench}
Francesco Croce, Maksym Andriushchenko, Vikash Sehwag, Edoardo Debenedetti,
  Nicolas Flammarion, Mung Chiang, Prateek Mittal, and Matthias Hein.
\newblock Robustbench: a standardized adversarial robustness benchmark.
\newblock \emph{arXiv preprint arXiv:2010.09670}, 2020.

\bibitem[D'Amour et~al.(2020)D'Amour, Heller, Moldovan, Adlam, Alipanahi,
  Beutel, Chen, Deaton, Eisenstein, Hoffman, et~al.]{d2020underspecification}
Alexander D'Amour, Katherine Heller, Dan Moldovan, Ben Adlam, Babak Alipanahi,
  Alex Beutel, Christina Chen, Jonathan Deaton, Jacob Eisenstein, Matthew~D
  Hoffman, et~al.
\newblock Underspecification presents challenges for credibility in modern
  machine learning.
\newblock \emph{arXiv preprint arXiv:2011.03395}, 2020.

\bibitem[Debenedetti et~al.(2022)Debenedetti, Sehwag, and
  Mittal]{debenedetti2022light}
Edoardo Debenedetti, Vikash Sehwag, and Prateek Mittal.
\newblock A light recipe to train robust vision transformers.
\newblock \emph{arXiv preprint arXiv:2209.07399}, 2022.

\bibitem[Defazio et~al.(2014)Defazio, Bach, and
  Lacoste-Julien]{defazio2014saga}
Aaron Defazio, Francis Bach, and Simon Lacoste-Julien.
\newblock Saga: A fast incremental gradient method with support for
  non-strongly convex composite objectives.
\newblock \emph{Advances in neural information processing systems}, 27, 2014.

\bibitem[Deng et~al.(2009)Deng, Dong, Socher, Li, Li, and
  Fei-Fei]{deng2009imagenet}
Jia Deng, Wei Dong, Richard Socher, Li-Jia Li, Kai Li, and Li~Fei-Fei.
\newblock Imagenet: A large-scale hierarchical image database.
\newblock In \emph{2009 IEEE conference on computer vision and pattern
  recognition}, pp.\  248--255. Ieee, 2009.

\bibitem[Dhariwal \& Nichol(2021)Dhariwal and Nichol]{dhariwal2021diffusion}
Prafulla Dhariwal and Alexander Nichol.
\newblock Diffusion models beat gans on image synthesis.
\newblock \emph{Advances in Neural Information Processing Systems},
  34:\penalty0 8780--8794, 2021.

\bibitem[Dosovitskiy et~al.(2020)Dosovitskiy, Beyer, Kolesnikov, Weissenborn,
  Zhai, Unterthiner, Dehghani, Minderer, Heigold, Gelly,
  et~al.]{dosovitskiy2020image}
Alexey Dosovitskiy, Lucas Beyer, Alexander Kolesnikov, Dirk Weissenborn,
  Xiaohua Zhai, Thomas Unterthiner, Mostafa Dehghani, Matthias Minderer, Georg
  Heigold, Sylvain Gelly, et~al.
\newblock An image is worth 16x16 words: Transformers for image recognition at
  scale.
\newblock \emph{arXiv preprint arXiv:2010.11929}, 2020.

\bibitem[Duda et~al.(1973)Duda, Hart, and Stork]{duda1973pattern}
Richard~O Duda, Peter~E Hart, and David~G Stork.
\newblock \emph{Pattern classification and scene analysis}, volume~3.
\newblock Wiley New York, 1973.

\bibitem[Edwards(2019)]{Edwards2019}
Chris Edwards.
\newblock Malevolent machine learning.
\newblock \emph{Commun. ACM}, 62\penalty0 (12):\penalty0 13–15, nov 2019.
\newblock ISSN 0001-0782.

\bibitem[Engstrom et~al.(2019)Engstrom, Ilyas, Santurkar, Tsipras, Tran, and
  Madry]{engstrom2019adversarial}
Logan Engstrom, Andrew Ilyas, Shibani Santurkar, Dimitris Tsipras, Brandon
  Tran, and Aleksander Madry.
\newblock Adversarial robustness as a prior for learned representations.
\newblock \emph{arXiv preprint arXiv:1906.00945}, 2019.

\bibitem[Esser et~al.(2021)Esser, Rombach, and Ommer]{esser2021taming}
Patrick Esser, Robin Rombach, and Bjorn Ommer.
\newblock Taming transformers for high-resolution image synthesis.
\newblock In \emph{Proceedings of the IEEE/CVF Conference on Computer Vision
  and Pattern Recognition}, pp.\  12873--12883, 2021.

\bibitem[Friedman et~al.(2001)Friedman, Hastie, Tibshirani,
  et~al.]{friedman2001elements}
Jerome Friedman, Trevor Hastie, Robert Tibshirani, et~al.
\newblock \emph{The elements of statistical learning}, volume~1.
\newblock Springer series in statistics New York, 2001.

\bibitem[Gatys et~al.(2015)Gatys, Ecker, and Bethge]{gatys2015texture}
Leon Gatys, Alexander~S Ecker, and Matthias Bethge.
\newblock Texture synthesis using convolutional neural networks.
\newblock \emph{Advances in neural information processing systems}, 28, 2015.

\bibitem[Gatys et~al.(2017)Gatys, Ecker, and Bethge]{gatys2017texture}
Leon~A Gatys, Alexander~S Ecker, and Matthias Bethge.
\newblock Texture and art with deep neural networks.
\newblock \emph{Current opinion in neurobiology}, 46:\penalty0 178--186, 2017.

\bibitem[Geirhos et~al.(2019)Geirhos, Rubisch, Michaelis, Bethge, Wichmann, and
  Brendel]{geirhos2018imagenettrained}
Robert Geirhos, Patricia Rubisch, Claudio Michaelis, Matthias Bethge, Felix~A.
  Wichmann, and Wieland Brendel.
\newblock Imagenet-trained {CNN}s are biased towards texture; increasing shape
  bias improves accuracy and robustness.
\newblock In \emph{International Conference on Learning Representations}, 2019.

\bibitem[Goodfellow et~al.(2015)Goodfellow, Shlens, and
  Szegedy]{goodfellow2015explaining}
Ian~J Goodfellow, Jonathon Shlens, and Christian Szegedy.
\newblock Explaining and harnessing adversarial examples (2014).
\newblock In \emph{International Conference on Learning Representations}, 2015.

\bibitem[Gulrajani \& Lopez-Paz(2020)Gulrajani and
  Lopez-Paz]{gulrajani2020search}
Ishaan Gulrajani and David Lopez-Paz.
\newblock In search of lost domain generalization.
\newblock \emph{arXiv preprint arXiv:2007.01434}, 2020.

\bibitem[Hazami et~al.(2022)Hazami, Mama, and
  Thurairatnam]{hazami2022efficient}
Louay Hazami, Rayhane Mama, and Ragavan Thurairatnam.
\newblock Efficient-vdvae: Less is more.
\newblock \emph{arXiv preprint arXiv:2203.13751}, 2022.

\bibitem[He et~al.(2016{\natexlab{a}})He, Zhang, Ren, and Sun]{he2016deep}
Kaiming He, Xiangyu Zhang, Shaoqing Ren, and Jian Sun.
\newblock Deep residual learning for image recognition.
\newblock In \emph{Proceedings of the IEEE conference on computer vision and
  pattern recognition}, pp.\  770--778, 2016{\natexlab{a}}.

\bibitem[He et~al.(2016{\natexlab{b}})He, Zhang, Ren, and Sun]{he2016identity}
Kaiming He, Xiangyu Zhang, Shaoqing Ren, and Jian Sun.
\newblock Identity mappings in deep residual networks.
\newblock In \emph{European conference on computer vision}, pp.\  630--645.
  Springer, 2016{\natexlab{b}}.

\bibitem[Hendrycks \& Dietterich(2019)Hendrycks and
  Dietterich]{hendrycks2019benchmarking}
Dan Hendrycks and Thomas Dietterich.
\newblock Benchmarking neural network robustness to common corruptions and
  perturbations.
\newblock \emph{arXiv preprint arXiv:1903.12261}, 2019.

\bibitem[Hendrycks et~al.(2019)Hendrycks, Mu, Cubuk, Zoph, Gilmer, and
  Lakshminarayanan]{hendrycks2019augmix}
Dan Hendrycks, Norman Mu, Ekin~D Cubuk, Barret Zoph, Justin Gilmer, and Balaji
  Lakshminarayanan.
\newblock Augmix: A simple data processing method to improve robustness and
  uncertainty.
\newblock \emph{arXiv preprint arXiv:1912.02781}, 2019.

\bibitem[Hendrycks et~al.(2021{\natexlab{a}})Hendrycks, Basart, Mu, Kadavath,
  Wang, Dorundo, Desai, Zhu, Parajuli, Guo, et~al.]{hendrycks2021many}
Dan Hendrycks, Steven Basart, Norman Mu, Saurav Kadavath, Frank Wang, Evan
  Dorundo, Rahul Desai, Tyler Zhu, Samyak Parajuli, Mike Guo, et~al.
\newblock The many faces of robustness: A critical analysis of
  out-of-distribution generalization.
\newblock In \emph{Proceedings of the IEEE/CVF International Conference on
  Computer Vision}, pp.\  8340--8349, 2021{\natexlab{a}}.

\bibitem[Hendrycks et~al.(2021{\natexlab{b}})Hendrycks, Zhao, Basart,
  Steinhardt, and Song]{hendrycks2021natural}
Dan Hendrycks, Kevin Zhao, Steven Basart, Jacob Steinhardt, and Dawn Song.
\newblock Natural adversarial examples.
\newblock In \emph{Proceedings of the IEEE/CVF Conference on Computer Vision
  and Pattern Recognition}, pp.\  15262--15271, 2021{\natexlab{b}}.

\bibitem[Howard et~al.(2017)Howard, Zhu, Chen, Kalenichenko, Wang, Weyand,
  Andreetto, and Adam]{howard2017mobilenets}
Andrew~G Howard, Menglong Zhu, Bo~Chen, Dmitry Kalenichenko, Weijun Wang,
  Tobias Weyand, Marco Andreetto, and Hartwig Adam.
\newblock Mobilenets: Efficient convolutional neural networks for mobile vision
  applications.
\newblock \emph{arXiv preprint arXiv:1704.04861}, 2017.

\bibitem[Huang et~al.(2017)Huang, Liu, Van Der~Maaten, and
  Weinberger]{huang2017densely}
Gao Huang, Zhuang Liu, Laurens Van Der~Maaten, and Kilian~Q Weinberger.
\newblock Densely connected convolutional networks.
\newblock In \emph{Proceedings of the IEEE conference on computer vision and
  pattern recognition}, pp.\  4700--4708, 2017.

\bibitem[Kingma \& Welling(2013)Kingma and Welling]{kingma2013auto}
Diederik~P Kingma and Max Welling.
\newblock Auto-encoding variational bayes.
\newblock \emph{arXiv preprint arXiv:1312.6114}, 2013.

\bibitem[Koh et~al.(2021)Koh, Sagawa, Xie, Zhang, Balsubramani, Hu, Yasunaga,
  Phillips, Gao, Lee, et~al.]{koh2021wilds}
Pang~Wei Koh, Shiori Sagawa, Sang~Michael Xie, Marvin Zhang, Akshay
  Balsubramani, Weihua Hu, Michihiro Yasunaga, Richard~Lanas Phillips, Irena
  Gao, Tony Lee, et~al.
\newblock Wilds: A benchmark of in-the-wild distribution shifts.
\newblock In \emph{International Conference on Machine Learning}, pp.\
  5637--5664. PMLR, 2021.

\bibitem[Krizhevsky et~al.(2009)Krizhevsky, Hinton,
  et~al.]{krizhevsky2009learning}
Alex Krizhevsky, Geoffrey Hinton, et~al.
\newblock Learning multiple layers of features from tiny images.
\newblock 2009.

\bibitem[LeCun et~al.(1998)LeCun, Bottou, Bengio, and
  Haffner]{lecun1998gradient}
Yann LeCun, L{\'e}on Bottou, Yoshua Bengio, and Patrick Haffner.
\newblock Gradient-based learning applied to document recognition.
\newblock \emph{Proceedings of the IEEE}, 86\penalty0 (11):\penalty0
  2278--2324, 1998.

\bibitem[Li et~al.(2017)Li, Yang, Song, and Hospedales]{li2017deeper}
Da~Li, Yongxin Yang, Yi-Zhe Song, and Timothy~M Hospedales.
\newblock Deeper, broader and artier domain generalization.
\newblock In \emph{Proceedings of the IEEE international conference on computer
  vision}, pp.\  5542--5550, 2017.

\bibitem[Liu et~al.(2023)Liu, Park, Azadi, Zhang, Chopikyan, Hu, Shi, Rohrbach,
  and Darrell]{liu2023more}
Xihui Liu, Dong~Huk Park, Samaneh Azadi, Gong Zhang, Arman Chopikyan, Yuxiao
  Hu, Humphrey Shi, Anna Rohrbach, and Trevor Darrell.
\newblock More control for free! image synthesis with semantic diffusion
  guidance.
\newblock In \emph{Proceedings of the IEEE/CVF Winter Conference on
  Applications of Computer Vision}, pp.\  289--299, 2023.

\bibitem[Liu et~al.(2021)Liu, Lin, Cao, Hu, Wei, Zhang, Lin, and
  Guo]{liu2021swin}
Ze~Liu, Yutong Lin, Yue Cao, Han Hu, Yixuan Wei, Zheng Zhang, Stephen Lin, and
  Baining Guo.
\newblock Swin transformer: Hierarchical vision transformer using shifted
  windows.
\newblock In \emph{Proceedings of the IEEE/CVF International Conference on
  Computer Vision}, pp.\  10012--10022, 2021.

\bibitem[Liu et~al.(2022)Liu, Mao, Wu, Feichtenhofer, Darrell, and
  Xie]{liu2022convnet}
Zhuang Liu, Hanzi Mao, Chao-Yuan Wu, Christoph Feichtenhofer, Trevor Darrell,
  and Saining Xie.
\newblock A convnet for the 2020s.
\newblock \emph{arXiv preprint arXiv:2201.03545}, 2022.

\bibitem[Locatello et~al.(2019)Locatello, Bauer, Lucic, Raetsch, Gelly,
  Sch{\"o}lkopf, and Bachem]{locatello2019challenging}
Francesco Locatello, Stefan Bauer, Mario Lucic, Gunnar Raetsch, Sylvain Gelly,
  Bernhard Sch{\"o}lkopf, and Olivier Bachem.
\newblock Challenging common assumptions in the unsupervised learning of
  disentangled representations.
\newblock In \emph{international conference on machine learning}, pp.\
  4114--4124. PMLR, 2019.

\bibitem[Madry et~al.(2017)Madry, Makelov, Schmidt, Tsipras, and
  Vladu]{madry2017towards}
Aleksander Madry, Aleksandar Makelov, Ludwig Schmidt, Dimitris Tsipras, and
  Adrian Vladu.
\newblock Towards deep learning models resistant to adversarial attacks.
\newblock \emph{arXiv preprint arXiv:1706.06083}, 2017.

\bibitem[Mahajan et~al.(2019)Mahajan, Tan, and Sharma]{mahajan2019preserving}
Divyat Mahajan, Chenhao Tan, and Amit Sharma.
\newblock Preserving causal constraints in counterfactual explanations for
  machine learning classifiers.
\newblock \emph{arXiv preprint arXiv:1912.03277}, 2019.

\bibitem[Mao et~al.(2022{\natexlab{a}})Mao, Geng, Yang, Wang, and
  Vondrick]{mao2022understanding}
Chengzhi Mao, Scott Geng, Junfeng Yang, Xin Wang, and Carl Vondrick.
\newblock Understanding zero-shot adversarial robustness for large-scale
  models.
\newblock \emph{arXiv preprint arXiv:2212.07016}, 2022{\natexlab{a}}.

\bibitem[Mao et~al.(2022{\natexlab{b}})Mao, Chen, Duan, Zhu, Qi, Li, Zhang,
  Xue, et~al.]{mao2022enhance}
Xiaofeng Mao, Yuefeng Chen, Ranjie Duan, Yao Zhu, Gege Qi, Xiaodan Li, Rong
  Zhang, Hui Xue, et~al.
\newblock Enhance the visual representation via discrete adversarial training.
\newblock \emph{Advances in Neural Information Processing Systems},
  35:\penalty0 7520--7533, 2022{\natexlab{b}}.

\bibitem[Marcos et~al.(2016)Marcos, Volpi, and Tuia]{marcos2016learning}
Diego Marcos, Michele Volpi, and Devis Tuia.
\newblock Learning rotation invariant convolutional filters for texture
  classification.
\newblock In \emph{2016 23rd International Conference on Pattern Recognition
  (ICPR)}, pp.\  2012--2017. IEEE, 2016.

\bibitem[Menon et~al.(2022)Menon, Chandratreya, and Vondrick]{menon2022task}
Sachit Menon, Ishaan~Preetam Chandratreya, and Carl Vondrick.
\newblock Task bias in vision-language models.
\newblock \emph{arXiv preprint arXiv:2212.04412}, 2022.

\bibitem[Muandet et~al.(2013)Muandet, Balduzzi, and
  Sch{\"o}lkopf]{muandet2013domain}
Krikamol Muandet, David Balduzzi, and Bernhard Sch{\"o}lkopf.
\newblock Domain generalization via invariant feature representation.
\newblock In \emph{International Conference on Machine Learning}, pp.\  10--18,
  2013.

\bibitem[Nichol \& Dhariwal(2021)Nichol and Dhariwal]{nichol2021improved}
Alexander~Quinn Nichol and Prafulla Dhariwal.
\newblock Improved denoising diffusion probabilistic models.
\newblock In \emph{International Conference on Machine Learning}, pp.\
  8162--8171. PMLR, 2021.

\bibitem[Orhan(2019)]{orhan2019robustness}
A~Emin Orhan.
\newblock Robustness properties of facebook's resnext wsl models.
\newblock \emph{arXiv preprint arXiv:1907.07640}, 2019.

\bibitem[Pearl(2009)]{pearl2009causality}
Judea Pearl.
\newblock \emph{Causality}.
\newblock Cambridge university press, 2009.

\bibitem[Radford et~al.(2021)Radford, Kim, Hallacy, Ramesh, Goh, Agarwal,
  Sastry, Askell, Mishkin, Clark, et~al.]{radford2021learning}
Alec Radford, Jong~Wook Kim, Chris Hallacy, Aditya Ramesh, Gabriel Goh,
  Sandhini Agarwal, Girish Sastry, Amanda Askell, Pamela Mishkin, Jack Clark,
  et~al.
\newblock Learning transferable visual models from natural language
  supervision.
\newblock \emph{arXiv preprint arXiv:2103.00020}, 2021.

\bibitem[Recht et~al.(2019)Recht, Roelofs, Schmidt, and
  Shankar]{recht2019imagenet}
Benjamin Recht, Rebecca Roelofs, Ludwig Schmidt, and Vaishaal Shankar.
\newblock Do imagenet classifiers generalize to imagenet?
\newblock In \emph{International Conference on Machine Learning}, pp.\
  5389--5400. PMLR, 2019.

\bibitem[Rezende et~al.(2014)Rezende, Mohamed, and
  Wierstra]{rezende2014stochastic}
Danilo~Jimenez Rezende, Shakir Mohamed, and Daan Wierstra.
\newblock Stochastic backpropagation and approximate inference in deep
  generative models.
\newblock In \emph{International conference on machine learning}, pp.\
  1278--1286. PMLR, 2014.

\bibitem[Rozsa et~al.(2016)Rozsa, G{\"u}nther, and Boult]{rozsa2016accuracy}
Andras Rozsa, Manuel G{\"u}nther, and Terrance~E Boult.
\newblock Are accuracy and robustness correlated.
\newblock In \emph{2016 15th IEEE international conference on machine learning
  and applications (ICMLA)}, pp.\  227--232. IEEE, 2016.

\bibitem[Rusak et~al.()Rusak, Schott, Zimmermann, Bitterwolfb, Bringmann,
  Bethge, and Brendel]{rusak2020increasing}
Evgenia Rusak, Lukas Schott, Roland Zimmermann, Julian Bitterwolfb, Oliver
  Bringmann, Matthias Bethge, and Wieland Brendel.
\newblock Increasing the robustness of dnns against im-age corruptions by
  playing the game of noise.

\bibitem[Russakovsky et~al.(2015)Russakovsky, Deng, Su, Krause, Satheesh, Ma,
  Huang, Karpathy, Khosla, Bernstein, et~al.]{russakovsky2015imagenet}
Olga Russakovsky, Jia Deng, Hao Su, Jonathan Krause, Sanjeev Satheesh, Sean Ma,
  Zhiheng Huang, Andrej Karpathy, Aditya Khosla, Michael Bernstein, et~al.
\newblock Imagenet large scale visual recognition challenge.
\newblock \emph{International journal of computer vision}, 115\penalty0
  (3):\penalty0 211--252, 2015.

\bibitem[Salimans et~al.(2016)Salimans, Goodfellow, Zaremba, Cheung, Radford,
  and Chen]{salimans2016improved}
Tim Salimans, Ian Goodfellow, Wojciech Zaremba, Vicki Cheung, Alec Radford, and
  Xi~Chen.
\newblock Improved techniques for training gans.
\newblock \emph{Advances in neural information processing systems}, 29, 2016.

\bibitem[Santurkar et~al.(2019)Santurkar, Tsipras, Tran, Ilyas, Engstrom, and
  Madry]{santurkar2019image}
Shibani Santurkar, Dimitris Tsipras, Brandon Tran, Andrew Ilyas, Logan
  Engstrom, and Aleksander Madry.
\newblock Image synthesis with a single (robust) classifier.
\newblock \emph{arXiv preprint arXiv:1906.09453}, 2019.

\bibitem[Sauer et~al.(2022)Sauer, Schwarz, and Geiger]{sauer2022stylegan}
Axel Sauer, Katja Schwarz, and Andreas Geiger.
\newblock Stylegan-xl: Scaling stylegan to large diverse datasets.
\newblock In \emph{ACM SIGGRAPH 2022 Conference Proceedings}, pp.\  1--10,
  2022.

\bibitem[Simonyan \& Zisserman(2014)Simonyan and Zisserman]{simonyan2014very}
Karen Simonyan and Andrew Zisserman.
\newblock Very deep convolutional networks for large-scale image recognition.
\newblock \emph{arXiv preprint arXiv:1409.1556}, 2014.

\bibitem[Singh et~al.(2023)Singh, Croce, and Hein]{singh2023revisiting}
Naman~D Singh, Francesco Croce, and Matthias Hein.
\newblock Revisiting adversarial training for imagenet: Architectures, training
  and generalization across threat models.
\newblock \emph{arXiv preprint arXiv:2303.01870}, 2023.

\bibitem[Sohl-Dickstein et~al.(2015)Sohl-Dickstein, Weiss, Maheswaranathan, and
  Ganguli]{sohl2015deep}
Jascha Sohl-Dickstein, Eric Weiss, Niru Maheswaranathan, and Surya Ganguli.
\newblock Deep unsupervised learning using nonequilibrium thermodynamics.
\newblock In \emph{International Conference on Machine Learning}, pp.\
  2256--2265. PMLR, 2015.

\bibitem[Szegedy et~al.(2013)Szegedy, Zaremba, Sutskever, Bruna, Erhan,
  Goodfellow, and Fergus]{szegedy2013intriguing}
Christian Szegedy, Wojciech Zaremba, Ilya Sutskever, Joan Bruna, Dumitru Erhan,
  Ian Goodfellow, and Rob Fergus.
\newblock Intriguing properties of neural networks.
\newblock \emph{arXiv preprint arXiv:1312.6199}, 2013.

\bibitem[Szegedy et~al.(2015)Szegedy, Liu, Jia, Sermanet, Reed, Anguelov,
  Erhan, Vanhoucke, and Rabinovich]{szegedy2015going}
Christian Szegedy, Wei Liu, Yangqing Jia, Pierre Sermanet, Scott Reed, Dragomir
  Anguelov, Dumitru Erhan, Vincent Vanhoucke, and Andrew Rabinovich.
\newblock Going deeper with convolutions.
\newblock In \emph{Proceedings of the IEEE conference on computer vision and
  pattern recognition}, pp.\  1--9, 2015.

\bibitem[Tan \& Le(2019)Tan and Le]{tan2019efficientnet}
Mingxing Tan and Quoc Le.
\newblock Efficientnet: Rethinking model scaling for convolutional neural
  networks.
\newblock In \emph{International Conference on Machine Learning}, pp.\
  6105--6114. PMLR, 2019.

\bibitem[Taori et~al.(2020)Taori, Dave, Shankar, Carlini, Recht, and
  Schmidt]{taori2020measuring}
Rohan Taori, Achal Dave, Vaishaal Shankar, Nicholas Carlini, Benjamin Recht,
  and Ludwig Schmidt.
\newblock Measuring robustness to natural distribution shifts in image
  classification.
\newblock \emph{Advances in Neural Information Processing Systems},
  33:\penalty0 18583--18599, 2020.

\bibitem[Touvron et~al.(2021)Touvron, Cord, Douze, Massa, Sablayrolles, and
  J{\'e}gou]{touvron2021training}
Hugo Touvron, Matthieu Cord, Matthijs Douze, Francisco Massa, Alexandre
  Sablayrolles, and Herv{\'e} J{\'e}gou.
\newblock Training data-efficient image transformers \& distillation through
  attention.
\newblock In \emph{International Conference on Machine Learning}, pp.\
  10347--10357. PMLR, 2021.

\bibitem[Tsipras et~al.(2018)Tsipras, Santurkar, Engstrom, Turner, and
  Madry]{tsipras2018robustness}
Dimitris Tsipras, Shibani Santurkar, Logan Engstrom, Alexander Turner, and
  Aleksander Madry.
\newblock Robustness may be at odds with accuracy.
\newblock \emph{arXiv preprint arXiv:1805.12152}, 2018.

\bibitem[Wang et~al.(2019)Wang, Ge, Xing, and Lipton]{wang2019learning}
Haohan Wang, Songwei Ge, Eric~P Xing, and Zachary~C Lipton.
\newblock Learning robust global representations by penalizing local predictive
  power.
\newblock \emph{arXiv preprint arXiv:1905.13549}, 2019.

\bibitem[Wang et~al.(2020{\natexlab{a}})Wang, Huang, Wu, and
  Xing]{wang2020squared}
Haohan Wang, Zeyi Huang, Xindi Wu, and Eric~P Xing.
\newblock Squared $\ell_2 $ norm as consistency loss for leveraging augmented
  data to learn robust and invariant representations.
\newblock \emph{arXiv preprint arXiv:2011.13052}, 2020{\natexlab{a}}.

\bibitem[Wang et~al.(2020{\natexlab{b}})Wang, Wu, Huang, and
  Xing]{wang2020high}
Haohan Wang, Xindi Wu, Zeyi Huang, and Eric~P Xing.
\newblock High-frequency component helps explain the generalization of
  convolutional neural networks.
\newblock In \emph{Proceedings of the IEEE/CVF Conference on Computer Vision
  and Pattern Recognition}, pp.\  8684--8694, 2020{\natexlab{b}}.

\bibitem[Wang et~al.(2021)Wang, Huang, Zhang, and Xing]{wang2021toward}
Haohan Wang, Zeyi Huang, Hanlin Zhang, and Eric Xing.
\newblock Toward learning human-aligned cross-domain robust models by
  countering misaligned features.
\newblock \emph{arXiv preprint arXiv:2111.03740}, 2021.

\bibitem[Wang et~al.(2022{\natexlab{a}})Wang, Zhang, and
  Sang]{wang2022fairclip}
Junyang Wang, Yi~Zhang, and Jitao Sang.
\newblock Fairclip: Social bias elimination based on attribute prototype
  learning and representation neutralization.
\newblock \emph{arXiv preprint arXiv:2210.14562}, 2022{\natexlab{a}}.

\bibitem[Wang et~al.(2022{\natexlab{b}})Wang, Li, Liu, Li, Tang, Xie, Chen, and
  Yu]{wang2022adaptive}
Yiqi Wang, Chaozhuo Li, Zheng Liu, Mingzheng Li, Jiliang Tang, Xing Xie, Lei
  Chen, and Philip~S Yu.
\newblock An adaptive graph pre-training framework for localized collaborative
  filtering.
\newblock \emph{ACM Transactions on Information Systems}, 41\penalty0
  (2):\penalty0 1--27, 2022{\natexlab{b}}.

\bibitem[Wang et~al.(2022{\natexlab{c}})Wang, Bai, Zhou, and Xie]{wang2022can}
Zeyu Wang, Yutong Bai, Yuyin Zhou, and Cihang Xie.
\newblock Can cnns be more robust than transformers?
\newblock \emph{arXiv preprint arXiv:2206.03452}, 2022{\natexlab{c}}.

\bibitem[Wightman(2019)]{rw2019timm}
Ross Wightman.
\newblock Pytorch image models.
\newblock \url{https://github.com/rwightman/pytorch-image-models}, 2019.

\bibitem[Wu et~al.(2023)Wu, Yuksekgonul, Zhang, and Zou]{wu2023discover}
Shirley Wu, Mert Yuksekgonul, Linjun Zhang, and James Zou.
\newblock Discover and cure: Concept-aware mitigation of spurious correlation.
\newblock \emph{arXiv preprint arXiv:2305.00650}, 2023.

\bibitem[Xiao et~al.(2023)Xiao, Tang, Wei, Liu, and Lin]{xiao2023masked}
Yao Xiao, Ziyi Tang, Pengxu Wei, Cong Liu, and Liang Lin.
\newblock Masked images are counterfactual samples for robust fine-tuning.
\newblock \emph{arXiv preprint arXiv:2303.03052}, 2023.

\bibitem[Xie et~al.(2020)Xie, Luong, Hovy, and Le]{xie2020self}
Qizhe Xie, Minh-Thang Luong, Eduard Hovy, and Quoc~V Le.
\newblock Self-training with noisy student improves imagenet classification.
\newblock In \emph{Proceedings of the IEEE/CVF conference on computer vision
  and pattern recognition}, pp.\  10687--10698, 2020.

\bibitem[Ye et~al.(2021)Ye, Li, Hong, Bai, Chen, Zhou, and Li]{ye2021ood}
Nanyang Ye, Kaican Li, Lanqing Hong, Haoyue Bai, Yiting Chen, Fengwei Zhou, and
  Zhenguo Li.
\newblock Ood-bench: Benchmarking and understanding out-of-distribution
  generalization datasets and algorithms.
\newblock \emph{arXiv preprint arXiv:2106.03721}, 2021.

\bibitem[Yu et~al.(2018)Yu, Wang, Shelhamer, and Darrell]{yu2018deep}
Fisher Yu, Dequan Wang, Evan Shelhamer, and Trevor Darrell.
\newblock Deep layer aggregation.
\newblock In \emph{Proceedings of the IEEE conference on computer vision and
  pattern recognition}, pp.\  2403--2412, 2018.

\bibitem[Yu et~al.(2022)Yu, Wang, Vasudevan, Yeung, Seyedhosseini, and
  Wu]{Yu2022CoCaCC}
Jiahui Yu, Zirui Wang, Vijay Vasudevan, Legg Yeung, Mojtaba Seyedhosseini, and
  Yonghui Wu.
\newblock Coca: Contrastive captioners are image-text foundation models.
\newblock 2022.

\bibitem[Yun et~al.(2019)Yun, Han, Oh, Chun, Choe, and Yoo]{yun2019cutmix}
Sangdoo Yun, Dongyoon Han, Seong~Joon Oh, Sanghyuk Chun, Junsuk Choe, and
  Youngjoon Yoo.
\newblock Cutmix: Regularization strategy to train strong classifiers with
  localizable features.
\newblock In \emph{Proceedings of the IEEE/CVF international conference on
  computer vision}, pp.\  6023--6032, 2019.

\bibitem[Zhang et~al.(2017)Zhang, Cisse, Dauphin, and
  Lopez-Paz]{zhang2017mixup}
Hongyi Zhang, Moustapha Cisse, Yann~N Dauphin, and David Lopez-Paz.
\newblock mixup: Beyond empirical risk minimization.
\newblock \emph{arXiv preprint arXiv:1710.09412}, 2017.

\bibitem[Zhang et~al.(2023)Zhang, Xu, Zhang, and Tao]{zhang2023vitaev2}
Qiming Zhang, Yufei Xu, Jing Zhang, and Dacheng Tao.
\newblock Vitaev2: Vision transformer advanced by exploring inductive bias for
  image recognition and beyond.
\newblock \emph{International Journal of Computer Vision}, pp.\  1--22, 2023.

\bibitem[Zhang(2019)]{zhang2019making}
Richard Zhang.
\newblock Making convolutional networks shift-invariant again.
\newblock In \emph{International conference on machine learning}, pp.\
  7324--7334. PMLR, 2019.

\bibitem[Zhang \& Zhu(2019)Zhang and Zhu]{zhang2019interpreting}
Tianyuan Zhang and Zhanxing Zhu.
\newblock Interpreting adversarially trained convolutional neural networks.
\newblock In \emph{International Conference on Machine Learning}, pp.\
  7502--7511. PMLR, 2019.

\bibitem[Zhao et~al.(2022)Zhao, Qu, Li, Yan, Liu, Li, Xie, and
  Tang]{zhao2022learning}
Jianan Zhao, Meng Qu, Chaozhuo Li, Hao Yan, Qian Liu, Rui Li, Xing Xie, and
  Jian Tang.
\newblock Learning on large-scale text-attributed graphs via variational
  inference.
\newblock \emph{arXiv preprint arXiv:2210.14709}, 2022.

\bibitem[Zhao et~al.(2023)Zhao, Li, Peng, Fang, Huang, Wang, Xie, and
  Gong]{zhao2023beyond}
Yi~Zhao, Chaozhuo Li, Jiquan Peng, Xiaohan Fang, Feiran Huang, Senzhang Wang,
  Xing Xie, and Jibing Gong.
\newblock Beyond the overlapping users: Cross-domain recommendation via
  adaptive anchor link learning.
\newblock In \emph{Proceedings of the 46th International ACM SIGIR Conference
  on Research and Development in Information Retrieval}, pp.\  1488--1497,
  2023.

\bibitem[Zhong et~al.(2020)Zhong, Zheng, Kang, Li, and Yang]{zhong2020random}
Zhun Zhong, Liang Zheng, Guoliang Kang, Shaozi Li, and Yi~Yang.
\newblock Random erasing data augmentation.
\newblock In \emph{Proceedings of the AAAI conference on artificial
  intelligence}, volume~34, pp.\  13001--13008, 2020.

\bibitem[Zhou et~al.(2022{\natexlab{a}})Zhou, Yu, Xie, Xiao, Anandkumar, Feng,
  and Alvarez]{zhou2022understanding}
Daquan Zhou, Zhiding Yu, Enze Xie, Chaowei Xiao, Animashree Anandkumar, Jiashi
  Feng, and Jose~M Alvarez.
\newblock Understanding the robustness in vision transformers.
\newblock In \emph{International Conference on Machine Learning}, pp.\
  27378--27394. PMLR, 2022{\natexlab{a}}.

\bibitem[Zhou et~al.(2022{\natexlab{b}})Zhou, LAI, and
  Jiang]{zhou2022vlstereoset}
Kankan Zhou, Yibin LAI, and Jing Jiang.
\newblock Vlstereoset: A study of stereotypical bias in pre-trained
  vision-language models.
\newblock Association for Computational Linguistics, 2022{\natexlab{b}}.

\end{thebibliography}
\bibliographystyle{iclr2024_conference}

%%%%%%%%%%%%%%%%%%%%%%%%%%%%%%%%%%%%%%%%%%%%%%%%%%%%%%%%%%%%%%%%%%%%%%%%%%%%%%%
%%%%%%%%%%%%%%%%%%%%%%%%%%%%%%%%%%%%%%%%%%%%%%%%%%%%%%%%%%%%%%%%%%%%%%%%%%%%%%%
% APPENDIX
%%%%%%%%%%%%%%%%%%%%%%%%%%%%%%%%%%%%%%%%%%%%%%%%%%%%%%%%%%%%%%%%%%%%%%%%%%%%%%%
%%%%%%%%%%%%%%%%%%%%%%%%%%%%%%%%%%%%%%%%%%%%%%%%%%%%%%%%%%%%%%%%%%%%%%%%%%%%%%%
\newpage
\appendix
\onecolumn

\section{Distribution Estimation}
\label{sec:distesti}
\subsection{Problem Formulation}

Given an unknown ground truth distribution: 
\begin{equation}
    \mathbf{P}= \textnormal{Unknown}(\mu,\Sigma)
\end{equation}
where $\mu \in \mathbb{R}^{p}$ and $\Sigma \in \mathbb{R}^{p\times p}$.

All the samples in our study are sampled from this distribution. 

We use $\X_k$ to denote the $k$\textsuperscript{th} dataset, with $n_k$ samples, and we use $\x_{k, i}$ to denote the $i$\textsuperscript{th} sample in it. 

We aim to consider the estimation 
of $\mu$ from two different models. 
The conventional smaller model which operates on 
only one dataset, 
and WLOG, we assume the smaller model works on $\X_0$;
and the bigger, zoo of CLIP-style models, 
which operates on a collection of datasets, 
we say it works on $m$ datasets, 
i.e., $\{\X_1, \X_2, \X_3, \dots, \X_m \}$, 
we will compare 
$\mathbb{E}[\widehat{\mu_0} - \mu]$ and $\mathbb{E}[\widehat{\mu_\textnormal{CLIP}} - \mu]$,  $\textnormal{VAR}(\widehat{\mu_0})$ and $\textnormal{VAR}(\widehat{\mu_\textnormal{CLIP}})$, $\mathbb{E}[\widehat{\Sigma_0} - \Sigma]$ and $\mathbb{E}[\widehat{\Sigma_\textnormal{CLIP}} - \Sigma]$ and $\textnormal{VAR}(\widehat{\Sigma_0})$ and $\textnormal{VAR}(\widehat{\Sigma_\textnormal{CLIP}})$.

\paragraph{Assumption I}
Due to dataset collection bias, we assume that, 
while all the data are sampled with the fixed distribution above, 
the bias of dataset collection will introduce a bias in 
the estimation of the true parameter $\mu$,
therefore
\begin{align}
    \widehat{\mu_i} = \mu + \epsilon_i
\end{align}
where 
\begin{align}
    \widehat{\mu_i} := \dfrac{1}{n_i}\sum^{n_i}_{j} \x_{i,j}
\end{align}
and 
\begin{align}
    \epsilon_i \sim N(\mathbf{0}, \mathbf{I})
\end{align}

\paragraph{Assumption II} 
\label{assp:2}
Due to dataset collection bias, we assume that, 
while all the data are sampled with the fixed distribution above, 
the bias of dataset collection will introduce a bias in 
the estimation of the true parameter $\Sigma$,
therefore
\begin{align}
    \widehat{\Sigma_i} = \epsilon^{'}_{i}\Sigma
\end{align}
where
\begin{align}
\widehat{\Sigma_i} := \dfrac{1}{n_i}\sum^{n_i}_j [(\x_{i,j}-\widehat{\mu_i})^T(\x_{i,j}-\widehat{\mu_i})] 
\end{align}
and
\begin{align}
    \epsilon^{'}_{i} \sim \textnormal{Exp}(\mathbf{1}),
\end{align}
% to make sure for any $\epsilon^{'}_{i}$, we have $\epsilon^{'}_{i} >0$ and $\mathbb{E}[\epsilon^{'}_{i}]=\mathbf{1}$.

\begin{proposition}
Under Assumptions I and II, we have estimators
\begin{align*}
    &\mathbb{E}[\widehat{\mu_\textnormal{CLIP}} - \mu] = \mathbb{E}[\widehat{\mu_0} - \mu], 
    &\mathbb{E}[\widehat{\Sigma_\textnormal{CLIP}} - \Sigma]
    = \mathbb{E}[\widehat{\Sigma_0} - \Sigma]\\
    &\textnormal{VAR}(\widehat{\mu_\textnormal{CLIP}}) \leq \textnormal{VAR}(\widehat{\mu_0}), 
    &\textnormal{VAR}(\widehat{\Sigma_\textnormal{CLIP}}) \leq
    \textnormal{VAR}(\widehat{\Sigma_0})
\end{align*}
where $\leq$ holds element-wise. 
\end{proposition}

\begin{proof}
\textbf{Estimation of $\mu$.}
Under Assumptions I and II, we have
% If we use Maximum Likelihood Estimation (MLE)~\citep{casella2021statistical} as the estimator, 
% then, we will have 
\begin{align}
    \widehat{\mu_0} = \mu + \epsilon_0
\end{align}

We can obtain $\mathbb{E}[\widehat{\mu_0} - \mu]$ 
and $\textnormal{VAR}(\widehat{\mu_0})$
by marginalizing out the randomness introduced by $\epsilon$:

\begin{align}
    \mathbb{E}[\widehat{\mu_0} - \mu] = 
    \mathbb{E}[\mu + \epsilon_0-\mu] = \mathbb{E}[\epsilon_0] = \mathbf{0}. 
\end{align}

\begin{align}
    \textnormal{VAR}(\widehat{\mu_0}) &= \mathbb{E}[\widehat{\mu_0}^2]-
    \mathbb{E}^2[\widehat{\mu_0}]  \nonumber \\
    & = \mathbb{E}[(\mu + \epsilon_0)^2] - \mathbb{E}^{2}[(\mu + \epsilon_0)]  \nonumber \\
     & = \mathbb{E}[\mu^2 + 2\mu\epsilon_0 + \epsilon_0^2] - (\mu+\mathbb{E}[\epsilon_0])^2 \nonumber \\
    & = \mathbb{E}[\epsilon_0^2] - \mathbb{E}^2[\epsilon_0]    \nonumber \\
    & = \textnormal{VAR}(\epsilon_0) \nonumber \\
    & = \mathbf{I}
\end{align}

For $\mathbb{E}[\widehat{\mu_\textnormal{CLIP}} - \mu]$ and $\textnormal{VAR}(\widehat{\mu_\textnormal{CLIP}})$, we have:

\begin{align}
    \mathbb{E}[\widehat{\mu_\textnormal{CLIP}} - \mu] = 
    \mathbb{E}[\dfrac{\sum_i^m\epsilon_in_i}{\sum_i^mn_i}] 
    = \dfrac{\sum_i^m\mathbb{E}[\epsilon_i]n_i}{\sum_i^mn_i}
    = \mathbf{0}.
\end{align}

and

\begin{align}
    \textnormal{VAR}(\widehat{\mu_\textnormal{CLIP}}) &= \mathbb{E}[\widehat{\mu_\textnormal{CLIP}}^2]-
    \mathbb{E}^2[\widehat{\mu_\textnormal{CLIP}}]  \nonumber \\
    & = \mathbb{E}[(\mu + \epsilon_\textnormal{CLIP})^2] - \mathbb{E}^{2}[(\mu + \epsilon_\textnormal{CLIP})] \nonumber  \\
     & = \mathbb{E}[\mu^2 + 2\mu\epsilon_\textnormal{CLIP} + \epsilon_\textnormal{CLIP}^2] - (\mu+\mathbb{E}[\epsilon_\textnormal{CLIP}])^2 \nonumber \\
    & = \mathbb{E}[\epsilon_\textnormal{CLIP}^2] - \mathbb{E}^2[\epsilon_\textnormal{CLIP}] 
\end{align}
Since $\mathbb{E}[\epsilon_\textnormal{CLIP}]=\mathbb{E}[\widehat{\mu_\textnormal{CLIP}} - \mu]=0$, we have:
\begin{align}
\textnormal{VAR}(\widehat{\mu_\textnormal{CLIP}})=\mathbb{E}[\epsilon_\textnormal{CLIP}^2]=\mathbb{E}[(\dfrac{\sum_i^m\epsilon_in_i}{\sum_i^mn_i})^2]
\end{align}
% {\color{red}
When we expand the square of sum, we will get the many squared terms (which are left finally) and many more that involves at least one $\mathbb{E}[\epsilon_i]\z$, where $\z$ can be any arbitrary stuff, and then since $\mathbb{E}[\epsilon_i] = \mathbf{0}$, $\z$ won't matter. Therefore, we have:
\begin{align}
\textnormal{VAR}(\widehat{\mu_\textnormal{CLIP}}) =
    \mathbb{E}[(\dfrac{\sum_i^m\epsilon_in_i}{\sum_i^mn_i})^2] 
    = \dfrac{\sum_i^mn_i^2\mathbb{E}[\epsilon_i^2]}{(\sum_i^mn_i)^2}
\end{align} 
Since $n_i \geq 1$ for $i=1,2,...,m$, we have $\sum_i^mn_i^2\leq (\sum_i^mn_i)^2$. 

Therefore,
\begin{align}
    \textnormal{VAR}(\widehat{\mu_\textnormal{CLIP}}) = \dfrac{\sum_i^mn_i^2\mathbb{E}[\epsilon_i^2]}{(\sum_i^mn_i)^2} \leq \dfrac{(\sum_i^mn_i)^2\mathbb{E}[\epsilon_i^2]}{(\sum_i^mn_i)^2} = \mathbb{E}[\epsilon_i^2]= \mathbf{I}
\end{align}
% }
% {\color{blue}
% \begin{align}
%     \textnormal{VAR}(\widehat{\mu_\textnormal{CLIP}}) = \dfrac{\sum_i^mn_i^2\mathbb{E}[(\epsilon_i-\mu)^2]}{(\sum_i^mn_i)^2} \leq \dfrac{(\sum_i^mn_i)^2\mathbb{E}[(\epsilon_i-\mu)^2]}{(\sum_i^mn_i)^2} = \mathbb{E}[(\epsilon_i-\mu)^2]= \mathbf{I}
% \end{align}
% }

\textbf{Estimation of $\Sigma$.}
We can obtain $\mathbb{E}[\widehat{\Sigma_0} - \Sigma]$ 
and $\textnormal{VAR}(\widehat{\Sigma_0})$
by marginalizing out the randomness introduced by $\epsilon^{'}$:
\begin{align}
    \mathbb{E}[\widehat{\Sigma_0} - \Sigma] = 
    \mathbb{E}[\epsilon^{'}_{0}\Sigma -\Sigma] = \mathbb{E}[(\epsilon^{'}_{0}-1)\Sigma]=\mathbb{E}[\epsilon^{'}_{0}-1]\mathbb{E}[\Sigma] = \mathbf{0}. 
\end{align}

\begin{align}
    \textnormal{VAR}(\widehat{\Sigma_0}) &= \mathbb{E}[\widehat{\Sigma_0}^2]-
    \mathbb{E}^2[\widehat{\Sigma_0}]  \nonumber \\
    & = \mathbb{E}[(\epsilon^{'}_{0}\Sigma)^2] - \mathbb{E}^{2}[\epsilon^{'}_{0}\Sigma]  \nonumber \\
     & = \mathbb{E}[\epsilon^{'2}_{0}\Sigma^2] - \mathbb{E}^2[\epsilon^{'}_{0}]\mathbb{E}^2[\Sigma] \nonumber \\
    & = \mathbb{E}[\epsilon^{'2}_{0}]\mathbb{E}[\Sigma^2] - \mathbb{E}^2[\epsilon^{'}_{0}]\mathbb{E}^2[\Sigma]  \nonumber \\
    & = 2\mathbb{E}[\Sigma^2]-\mathbb{E}^2[\Sigma]  \nonumber \\
    & = 2\Sigma^2-\Sigma^2 \nonumber \\
    & = \Sigma^2
\end{align}
For $\mathbb{E}[\widehat{\Sigma_\textnormal{CLIP}} - \Sigma]$ and $\textnormal{VAR}(\widehat{\Sigma_\textnormal{CLIP}})$, we have:
\begin{align}
    \mathbb{E}[\widehat{\Sigma_\textnormal{CLIP}} - \Sigma] = 
    \mathbb{E}[\epsilon^{'}_\textnormal{CLIP}\Sigma -\Sigma] = \mathbb{E}[(\epsilon^{'}_\textnormal{CLIP}-1)\Sigma]=\mathbb{E}[\epsilon^{'}_\textnormal{CLIP}-1]\mathbb{E}[\Sigma] = \mathbf{0}. 
\end{align}
\begin{align}
\label{eq:varclip}
    \textnormal{VAR}(\widehat{\Sigma_\textnormal{CLIP}}) &= \mathbb{E}[\widehat{\Sigma_\textnormal{CLIP}}^2]-
    \mathbb{E}^2[\widehat{\Sigma_\textnormal{CLIP}}] \nonumber  \\
    & = \mathbb{E}[(\epsilon^{'}_\textnormal{CLIP}\Sigma)^2] - \mathbb{E}^{2}[\epsilon^{'}_\textnormal{CLIP}\Sigma] \nonumber  \\
     & = \mathbb{E}[\epsilon^{'2}_\textnormal{CLIP}\Sigma^2] - \mathbb{E}^2[\epsilon^{'}_\textnormal{CLIP}]\mathbb{E}^2[\Sigma] \nonumber \\
    & = \mathbb{E}[\epsilon^{'2}_\textnormal{CLIP}]\mathbb{E}[\Sigma^2] - \mathbb{E}^2[\Sigma]
\end{align}
Consider that 
\begin{align}
\widehat{\Sigma_\textnormal{CLIP}}=\dfrac{\Sigma^{m}_{i}\Sigma^{n_i}_{j}(\x_{i,j}-\widehat{\mu_\textnormal{CLIP}})^2}{\Sigma^{m}_{i}n_{i}}=
\dfrac{\Sigma^{m}_{i}\epsilon^{2}_\textnormal{CLIP}n_{i}}{\Sigma^{m}_{i}n_{i}}=\epsilon^{'}_\textnormal{CLIP}\Sigma
\end{align}
We will have:
\begin{align}
    \epsilon^{'}_\textnormal{CLIP}=\dfrac{\epsilon^{2}_\textnormal{CLIP}}{\Sigma}
\end{align}
Thus, we have $\epsilon^{'2}_\textnormal{CLIP}=\dfrac{\epsilon^{4}_\textnormal{CLIP}}{\Sigma^2}$.
Next, we will compute $\mathbb{E}[\epsilon^{'2}_\textnormal{CLIP}]$ as follows:
\begin{align}
\label{eq:Ee'2}
    \mathbb{E}[\epsilon^{'2}_\textnormal{CLIP}]&=\mathbb{E}[\dfrac{\epsilon^{4}_\textnormal{CLIP}}{\Sigma^2}] \nonumber \\
    &= \dfrac{\mathbb{E}[\epsilon^{4}_\textnormal{CLIP}]}{\Sigma^2}
\end{align}
By definition, we have:
\begin{align}  \epsilon_\textnormal{CLIP}=\dfrac{\Sigma^m_{i}\Sigma^{n_{i}}_{j}(\x_{i,j}-\widehat{\mu_{i}})}{\Sigma^{m}_{i}n_{i}}
\end{align}
Therefore,
\begin{align}
\epsilon^2_\textnormal{CLIP}=\dfrac{(\Sigma^m_{i}\Sigma^{n_{i}}_{j}(\x_{i,j}-\widehat{\mu_{i}}))^2}{(\Sigma^{m}_{i}n_{i})^2}
\end{align}
As the value of $x_{i,j}-\widehat{\mu_i}$ can be either positive or negative, we have:
\begin{align}
    \epsilon^2_\textnormal{CLIP} \leq \dfrac{\Sigma^m_{i}\Sigma^{n_{i}}_{j}(\x_{i,j}-\widehat{\mu_{i}})^2}{\Sigma^{m}_{i}n_{i}}\dfrac{1}{\Sigma^m_i{n_i}}
\end{align}
Since both $(x_{i,j}-\widehat{\mu_i})^2$ and $n_i$ are positive values, we further have:
\begin{align}
    \epsilon^2_\textnormal{CLIP} \leq \Sigma^m_{i}\dfrac{\Sigma^{n_i}_{j}(\x_{i,j}-\widehat{\mu_{i}})^2)}{n_i}\dfrac{1}{\Sigma^m_i{n_i}}  
    =\dfrac{\Sigma^m_{i}\widehat{\Sigma_i}}{\Sigma^m_i{n_i}}   
    =\dfrac{\Sigma^m_{i}\epsilon^{'}_{i}\Sigma}{\Sigma^m_i{n_i}}
\end{align}
Thus, we can obtain
\begin{align}
    \epsilon^4_\textnormal{CLIP}\leq \dfrac{(\Sigma^m_{i}\epsilon^{'}_{i}\Sigma)^2}{(\Sigma^m_i{n_i})^2} = \dfrac{(\Sigma^m_{i}\epsilon^{'}_{i})^2\Sigma^2}{(\Sigma^m_i{n_i})^2}
\end{align}
Therefore, we have:
\begin{align}
\label{eq:epsilon4}
    \mathbb{E}[\epsilon^4_\textnormal{CLIP}] \leq \mathbb{E}[\dfrac{(\Sigma^m_{i}\epsilon^{'}_{i})^2\Sigma^2}{(\Sigma^m_i{n_i})^2}] = \dfrac{\mathbb{E}[(\Sigma^m_{i}\epsilon^{'}_{i})^2]\Sigma^2}{(\Sigma^m_i{n_i})^2}
\end{align}
By Assumption~\ref{assp:2}, $\epsilon^{'}_{i} \sim \textnormal{Exp}(\mathbf{1})$, we have $\mathbb{E}[\epsilon^{'}_{i}]=1$ and $\textnormal{VAR}(\epsilon^{'}_{i})=1$.

Since $\epsilon^{'}_{i}$ are independent with each other, we have:
\begin{align}
\label{eq:epsilon2}
    \mathbb{E}[(\Sigma^m_{i}\epsilon^{'}_{i})^2]&=\textnormal{VAR}(\Sigma^m_{i}\epsilon^{'}_{i})+\mathbb{E}^2[\Sigma^m_{i}\epsilon^{'}_{i}]\nonumber \\
    &=\Sigma^m_{i}\textnormal{VAR}(\epsilon^{'}_{i})+(\Sigma^m_{i}\mathbb{E}[\epsilon^{'}_{i}])^2 \nonumber  \\
    &=m+m^2
\end{align}
Substituting Eq.~\ref{eq:epsilon2} into Eq.~\ref{eq:epsilon4}, we have:
\begin{align}
    \mathbb{E}[\epsilon^4_\textnormal{CLIP}] \leq \dfrac{m+m^2}{(\Sigma^m_i{n_i})^2}\Sigma^2
\end{align}

Since $n_i \geq 1$ for $i=1,2,...,m$, we have $\Sigma^m_i{n_i} \geq m$ and $(\Sigma^m_i{n_i})^2 \geq m^2$.

Since $m \geq 1$, we have: $m^2 \geq m$.

Therefore,
\begin{align}
\label{eq:epsilon4final}
    \mathbb{E}[\epsilon^4_\textnormal{CLIP}] \leq \dfrac{m+m^2}{m^2}\Sigma^2 \leq \dfrac{2m^2}{m^2}\Sigma^2 = 2\Sigma^2
\end{align}
Substituting Eq.~\ref{eq:epsilon4final} into Eq.~\ref{eq:Ee'2}, we have:
\begin{align}
\label{eq:Ee'2final}
    \mathbb{E}[\epsilon^{'2}_\textnormal{CLIP}] \leq \dfrac{2\Sigma^2}{\Sigma^2}=2
\end{align}
Substituting Eq.~\ref{eq:Ee'2final} into Eq.~\ref{eq:varclip}, we have:
\begin{align}
    \textnormal{VAR}(\widehat{\Sigma_\textnormal{CLIP}}) = \mathbb{E}[\epsilon^{'2}_\textnormal{CLIP}]\mathbb{E}[\Sigma^2] - \mathbb{E}^2[\Sigma] \leq 2\mathbb{E}[\Sigma^2]-\mathbb{E}^2[\Sigma] = \mathbb{E}[\Sigma] = \Sigma
\end{align}

We summarize the above results as follows:
For conventional fixed dataset estimators, we have:
\begin{align}
    \mathbb{E}[\widehat{\mu_0} - \mu] = 
     \mathbf{0} \nonumber   \\
     \textnormal{VAR}(\widehat{\mu_0})=\mathbf{I} \nonumber \\
     \mathbb{E}[\widehat{\Sigma_0} - \Sigma] =\mathbf{0} \nonumber  \\
     \textnormal{VAR}(\widehat{\Sigma_0})=\Sigma^2 \nonumber
\end{align}

% \begin{itemize}[leftmargin=*]
% \item $\mathbb{E}[\widehat{\mu_0} - \mu] = 
%      \mathbf{0}$
% \item $\textnormal{VAR}(\widehat{\mu_0})=\mathbf{I}$
% \item $\mathbb{E}[\widehat{\Sigma_0} - \Sigma] =\mathbf{0}$
% \item $\textnormal{VAR}(\widehat{\Sigma_0})=\Sigma^2$
% \end{itemize}
For CLIP-style estimators, we have:
\begin{align}
    \mathbb{E}[\widehat{\mu_\textnormal{CLIP}} - \mu] = \mathbf{0} \nonumber    \\
    \textnormal{VAR}(\widehat{\mu_\textnormal{CLIP}}) \leq \mathbf{I} \nonumber  \\
    \mathbb{E}[\widehat{\Sigma_\textnormal{CLIP}} - \Sigma] = \mathbf{0} \nonumber  \\
    \textnormal{VAR}(\widehat{\Sigma_\textnormal{CLIP}}) \leq \Sigma, \nonumber
\end{align}
where $\leq$ holds element-wise.
% \begin{itemize}[leftmargin=*]
% \item $\mathbb{E}[\widehat{\mu_\textnormal{CLIP}} - \mu] = \mathbf{0}$
% \item $\textnormal{VAR}(\widehat{\mu_\textnormal{CLIP}}) \leq \mathbf{I}$
% \item $\mathbb{E}[\widehat{\Sigma_\textnormal{CLIP}} - \Sigma] = \mathbf{0}$
% \item $\textnormal{VAR}(\widehat{\Sigma_\textnormal{CLIP}}) \leq \Sigma$
% \end{itemize}
\end{proof}
 The results show that, both conventional estimator and zoo of CLIP-style estimator can recover the true $\mu, \Sigma$ of the unknown distribution, but zoo of CLIP-style estimator will have a lower variance, which is more stable to accomplish the task. This conclusion holds for any distributions.
 % as long as we know the closed form of the MLE for that distribution. 

 With these theoretical evidence, we kindly argue that biased towards the zoo of CLIP-style models is better than biased on conventional fixed datasets. In addition, recent advances in incorporating the foundation model into various tasks~\citep{liu2023more,zhang2023vitaev2,bose2023movieclip} also reveals that the community has utilized the foundation model on a large scale and pays little attention on these biases. 
 
\section{Notes on the Experimental Setup}
\subsection{Notes on Models}
Note that we only re-evaluate existing model checkpoints, and hence do not perform any hyperparameter tuning for evaluated models. Our model evaluations are done on 8 NVIDIA V100 GPUs. With our Sparsified VQGAN model, our method is also feasible to work with a small amount of GPU resources. As shown in Appendix~\ref{sec:sparsevqgan}, the proposed protocol can work on a single NVIDIA V100 GPU efficiently.

\subsection{Hyperparameter Tuning}
Our method is generally parameter-free except for the computation budget and perturbation step size. In our experiments, the computation budget is the maximum iteration number of Sparse VQGAN. We consider the predefined value to be 50, as it guarantees the degree of perturbation with acceptable time costs. We provide the experiment for step size configuration in Section~\ref{sec:ss}.

\section{In-depth Analysis on the Transformer Family}
\label{sec:multi-head}
In Table~\ref{tab:result:main}, we notice a large difference between the methods in the proposed FMR metric, even within the transformer family. After checking the distribution of misclassified perturbed images of different models, we find that these images are rather random and do not reveal any obvious "weak classes". One possible reason for this phenomenon may due to their internal architecture that are related to the self-attention (Self-Att) mechanism. Many current Vision Transformer architectures adopt a multi-head self-attention (MHSA) design where each head tends to focus on different object components. In some sense, MHSA can be interpreted as a mixture of information bottlenecks (IB) where the stacked Self-Att modules in Vision Transformers can be broadly regarded as an iterative repeat of the IB process which promotes grouping and noise filtering. More details of the connection between the Self-Att and IB can be found in (\citep{zhou2022understanding}, Sec.2.3). As revealed in~\citep{zhou2022understanding}, having more heads leads to improved expressivity and robustness. But the reduced channel number per head also causes decreased clean accuracy. The best trade-off is achieved with 32 channels per head.

Table~\ref{tab:ori_head_number} illustrates the head number configurations of various models employed in our experiment.

\begin{table*}[h]
\small 
\centering
\renewcommand{\arraystretch}{1.2}
\caption{Details of head numbers configurations.)}
\begin{tabular}{cc}
\hline
\textbf{Model} & \textbf{Head Number}   \\ \hline
ViT & 12  \\ 
DeiT & 12  \\ 
Twins & (3,6,12,24)  \\ 
Visformer & 6   \\ 
Swin & (4,8,16,32) \\ \hline
\end{tabular}
\label{tab:ori_head_number}
\end{table*}

Swin Transformer exhibits the highest number of heads among them. Despite its suboptimal accuracy on the standard dataset, it achieves the best FMR. This corroborates the finding in~\citep{zhou2022understanding} that increased head numbers enhance expressivity and robustness, albeit at the expense of clean accuracy.

To further verify the impact of head numbers, we trained Swin Transformer with varying head configurations and obtained the following results in Table~\ref{tab:head_number}.

\begin{table*}[h]
\small 
\centering
\renewcommand{\arraystretch}{1.2}
\caption{The performance of Swin Transformer with different head number configurations. We find that increased heads enhance expressivity and robustness.)}
\begin{tabular}{cccc}
\hline
\textbf{\#Params} & \textbf{Head Number}  & \textbf{SA} & \textbf{FMR} \\ \hline
88M & (2,4,8,16) & 80.82 & 64.85\\ 
88M & (3,6,12,24) & 81.98 & 67.48\\ 
88M & (4,8,16,32) & 81.67 & 69.73\\ 
88M & (5,10,20,40) & 81.05  & 69.97\\ \hline
\end{tabular}
\label{tab:head_number}
\end{table*}

With comparable numbers of parameters, we observe that their accuracies on the standard dataset are relatively similar. With the augmentation of head numbers, the FMR value also escalates, which validates our hypothesis that increased heads enhance expressivity and robustness.

\section{Transferability of Generated Images}
\label{sec:transfer}

We first study whether our generated images are model-specific, since the generation of the images involves the gradient of the original model. 
We train several architectures, namely 
EfficientNet \citep{tan2019efficientnet}, MobileNet \citep{howard2017mobilenets}, SimpleDLA \citep{yu2018deep}, VGG19 \citep{simonyan2014very}, PreActResNet \citep{he2016identity}, GoogLeNet \citep{szegedy2015going}, and DenseNet121 \citep{huang2017densely}
and test these models with the images that generated when testing ResNet. 
We also train another ResNet following the same procedure 
to check the transferability across different runs in one architecture. 

% \begin{wraptable}{R}{0.4\textwidth}
% \small 
% % \setlength{\abovecaptionskip}{0.2cm}
% % \setlength{\belowcaptionskip}{-0.2cm}
% \centering 
% \renewcommand{\arraystretch}{1.2}
% \begin{tabular}{cccc}
% \hline
% \textbf{Model} & \textbf{SA} & \textbf{CA} & \textbf{MA} \\ \hline
% ResNet & 95.38 & 51.67 & 67.49 \\ \hline
% EfficientNet & 91.37 & 62.57 & 70.24 \\
% MobileNet & 91.63 & 62.97 & 70.76 \\
% SimpleDLA & 92.25 & 61.03 & 69.75 \\
% VGG & 93.54 & 66.01 & 73.71  \\
% PreActResNet & 94.06 & 63.26 & 72.09 \\
% ResNet & 94.67 & 62.70 & 72.21 \\
% GoogLeNet & 95.06 & 63.39 & 72.82 \\
% DenseNet & 95.26 & 63.28 & 72.80  \\ \hline
% \end{tabular}
% \caption{Performances of transferability.}
% \label{tab:result:transfer}
% \end{wraptable}

% \begin{wraptable}{R}{0.4\textwidth}
% \small 
% % \setlength{\abovecaptionskip}{0.2cm}
% % \setlength{\belowcaptionskip}{-0.2cm}
% \centering 
% \renewcommand{\arraystretch}{1.2}
% \begin{tabular}{cccc}
% \hline
% \textbf{Model} & \textbf{SA}  & \textbf{MA} & \textbf{FMR}\\ \hline
% ResNet & 95.38  & 67.49 & 54.17\\ \hline
% EfficientNet & 91.37  & 70.24 & 68.48\\
% MobileNet & 91.63  & 70.76 & 68.72\\
% SimpleDLA & 92.25  & 69.75 & 66.16\\
% VGG & 93.54  & 73.71  & 70.57\\
% PreActResNet & 94.06  & 72.09 & 67.25\\
% ResNet & 94.67  & 72.21 & 66.23\\
% GoogLeNet & 95.06  & 72.82 & 66.68\\
% DenseNet & 95.26  & 72.80  & 66.43\\ \hline
% \end{tabular}
% \caption{Performances of transferability.}
% \label{tab:result:transfer}
% \end{wraptable}

\begin{wraptable}{R}{0.4\textwidth}
\small 
\centering 
\caption{Performances of transferability.}
\begin{tabular}{cccc}
\hline
\textbf{Model} & \textbf{SA} & \textbf{PA} & \textbf{FMR}\\ \hline
ResNet & 95.38  & 51.67 & 54.17\\ \hline
ResNet & 94.67 & 56.09 & 59.25\\
DenseNet & 94.26 & 60.48 & 64.17\\
SimpleDLA & 92.25 & 61.03 & 66.16\\
GoogLeNet & 92.06 & 61.10 & 66.38\\
PreActResNet & 90.91 & 61.14 & 67.25\\
EfficientNet & 91.37 & 62.57 & 68.48\\
MobileNet & 91.63 & 62.97 & 68.72\\
VGG & 93.54 & 66.01 & 70.57\\  \hline

\end{tabular}

\label{tab:result:transfer}
\end{wraptable}

\textbf{Transferability of the generated images.} Table~\ref{tab:result:transfer} shows a 
reasonable transferability of the generated images
as the FMR are all lower than the SA, 
although 
we can also observe an improvement over the FMR when tested in the new models. 
These results suggest that our method of generating images 
can be potentially used in a broader scope: 
we can also leverage the method to generate a static set of images and set a benchmark dataset to help the development of robustness methods. 

\textbf{Reliability of the FMR metric.} Moreover, these results contribute to the validation of the reliability of the FMR metric: given that each model's FMR gets computed using a different test set, it is not clear why FMR would be a reliable metric that can be used to compare two models. In this experiment, however, the models are tested using the same fixed test set that was initially generated during the evaluation on ResNet. Remarkably, the strong correlation observed between FMR and PA at the fixed test sets lends credence to the reliability of the FMR metric.

\textbf{New findings.}
In addition, our results might potentially help mitigate a debate on 
whether more accurate architectures are naturally more robust: 
on one hand, we have results showing that more accurate architectures indeed 
lead to better empirical performances on certain (usually fixed) robustness benchmarks \citep{rozsa2016accuracy,hendrycks2019benchmarking};
while on the other hand, 
some counterpoints suggest the higher robustness numerical performances are only because these models capture 
more non-robust features that also happen exist in the fixed benchmarks \citep{tsipras2018robustness,wang2020high,taori2020measuring}. 
Table~\ref{tab:result:transfer} show some examples to support the latter argument:
in particular, we notice that VGG, while ranked in the middle of the accuracy ladder, 
interestingly
stands out when tested with generated images. 
These results continue to support our argument that 
a dynamic robustness test scenario can help reveal more properties of the model.

\begin{wraptable}{R}{0.4\textwidth}
\small
\centering 
\setlength{\abovecaptionskip}{0.2cm}
\setlength{\belowcaptionskip}{-0.2cm}
\caption{Results on whether initiating with adversarial images ($\epsilon = 0.003$).}
\begin{tabular}{ccc}
\hline
 \textbf{Data} &\textbf{SA}  & \textbf{FMR}\\ \hline
 regular & 95.38  & 57.80\\
 w. FGSM   & 95.30  &  65.79\\ \hline
\end{tabular}

\label{tab:results:fgsm}
\end{wraptable}

\section{Initiating with Adversarial Attacked Images}
\label{sec:fgsm}
Since our method using the gradient of the evaluated model
reminds readers about the gradient-based attack methods in adversarial robustness literature, 
we test whether initiating the perturbation process 
with an adversarial example 
will further degrade the accuracy. 

We first generate the images with FGSM attack \citep{goodfellow2015explaining}.
Table~\ref{tab:results:fgsm}. 
shows that initiating with the FGSM adversarial examples
barely affect the FMR, 
which is probably because the 
major style-wise perturbation 
will erase the imperceptible perturbations 
the adversarial examples introduce.

\section{Adversarially Robust Models}
\label{sec:pgd}
With evidence suggesting
the adversarially robust models  
are considered more human perceptually aligned \citep{engstrom2019adversarial,zhang2019interpreting,wang2020high}, 
we compare the vanilla model to a model trained by PGD \citep{madry2017towards} ($\ell_\infty$ norm smaller than 0.03). 

\begin{wraptable}{R}{0.4\textwidth}
\small 
\centering 
\caption{Performances comparison with vanilla model and PGD trained model.}
\begin{tabular}{ccccc}
\hline
\textbf{Data} & \textbf{Model} & \textbf{SA}  & \textbf{FMR}\\ \hline
\multirow{2}{*}{Van.} & Van. & 95.38   & 57.79\\
 & PGD & 85.70  & 95.96\\ \hline
\multirow{2}{*}{PGD} & Van. & 95.38  & 62.41 \\
 & PGD & 85.70  & 66.18 \\ \hline
\end{tabular}

\label{tab:result:adv}
\end{wraptable}

As shown in Table~\ref{tab:result:adv},
adversarially trained model and vanillaly trained model indeed process the data differently:
the transferability of the generated images between these two regimes can barely hold.
In particular, the PGD model can almost maintain its performances when tested with the images generated by the vanilla model. 

However, despite the differences,  
the PGD model's robustness weak spots are exposed to a similar degree 
with the vanilla model by our test system: 
the FMR of the vanilla model and the PGD model 
are only 57.79 and 66.18, respectively.  
We believe this result can further help advocate our belief 
that the robustness test needs to be 
a dynamic process generating images conditioning on the model to test, 
and thus further help validate the importance of our contribution.

\section{Augmentation through Static Adversarial Training}
\label{sec:advt}

Intuitively, inspired by the success of adversarial training \citep{madry2017towards} in defending 
models against adversarial attacks, 
a natural method to improve the empirical performances under our new test protocol 
is to augment the training data with perturbed training images 
generated by the same process. 
We aim to validate the effectiveness of this method here. 
\begin{wraptable}{R}{0.4\textwidth}
\setlength{\abovecaptionskip}{0.2cm}
\small
\centering 
\caption{Test performances of the model trained in a vanilla manner (denoted as Van.) or with augmentation data offered through our approach (marked by the second column). We report two sets of performances, split by whether the perturbed images are generated according to the vanilla model or the augmented model (marked by the first column).}
\begin{tabular}{ccccc}
\hline
\textbf{Data} & \textbf{Model} & \textbf{SA}  & \textbf{FMR}  \\ \hline
\multirow{2}{*}{Van.} & Van. & 95.38  & 57.82  \\
 & Aug & 87.41   & 89.13  \\ \hline
\multirow{2}{*}{Aug.} & Van. & 95.38  & 58.03  \\
 & Aug & 87.41   & 78.61  \\
 \hline
\end{tabular}
     
\label{tab:result:aug}
\end{wraptable}

However, the computational load of generation process is not ideal to serve 
the standard adversarial training strategy, 
and we can only have one copy of the perturbed training samples. 
Fortunately, 
we notice that some recent advances in training with data augmentation  
can help learn robust representations with a limited scope of augmented samples \citep{wang2020squared}, 
which we use here.

% \begin{table}[]
% \small
% \centering 
% \begin{tabular}{ccccccc}
% \hline
% data & model & SA & CA & MA & CA/SA & MA/SA \\ \hline
% \multirow{5}{*}{Van.} & Van. & 95.38 & 55.12 & 67.49 & 0.578 & 0.708 \\
%  & Art & 87.41 & 77.88 & 78.63 & 0.891 & 0.900 \\
%  & Cart. & 85.98 & 77.95 & 77.50 & 0.907 & 0.901 \\
%  & Real. & 87.64 & 77.74 & 78.48 & 0.887 & 0.895 \\
%  & Sket. & 86.46 & 77.51 & 77.96 & 0.896 & 0.902 \\ \hline
% Mix & Van. & 95.38 & 71.13 & 78.08 & 0.746 & 0.819 \\
% Art & Art & 87.41 & 60.82 & 67.39 & 0.696 & 0.771 \\
% Cart. & Cart. & 85.98 & 59.53 & 60.46 & 0.692 & 0.703 \\
% Real. & Real. & 87.64 & 60.41 & 67.06 & 0.689 & 0.765 \\
% Sket. & Sket. & 86.46 & 59.46 & 66.57 & 0.688 & 0.770 \\ \hline
% \end{tabular}
% \caption{Test performances of the model trained in a vanilla manner (denoted as Van.) or with augmentation data offered through our approach (marked by the second column). We report two sets of performances, split by whether the counterfactual images are generated according to the vanilla model or the augmented model (marked by the first column). The scores are reported as an average over four styles. The detailed statistics are in appendix. \textcolor{blue}{\textbf{Need to be updated}}}
% \label{tab:result:aug}
% \end{table}

We report our results in Table~\ref{tab:result:aug}. The first thing we observe is that 
% robustness against regular adversarial examples generalize well to
% robustness against common corruptions.                 
the model trained with the augmentation data offered through our approach could preserve a relatively higher performance (FMR 89.13) when testing with the perturbed images generated according to the vanilla model. Since we have shown the perturbed samples have a reasonable transferability in the main manuscript, this result indicates the robustness we brought when training with the perturbed images generated by our approach. 

In addition, when tested with the perturbed images generated according to the augmented model, the augmented model displays a marked resilience (FMR 78.61) in the face of these perturbations compared with the model trained in a vanilla manner (FMR 58.03). Nevertheless, it is noteworthy that the augmented model's performance does exhibit a discernible decline under these circumstances, which once more underscores the efficacy of our approach.

\section{Grayscale Models}
\label{sec:greyscale}

Our previous visualization suggests that 
a shortcut the perturbed generation system can take  
is to significantly shift the color of the images, 
for which a grey-scale model should easily maintain the performance. 
Thus, we train a grayscale model by changing the ResNet input channel to be 1 and transforming the input images to be grayscale upon feeding into the model. 
We report the results in Table~\ref{tab:result:gray}.

Interestingly, we notice that the grayscale model 
cannot defend against the shift introduced by our system by ignoring the color information. 
On the contrary, 
it seems to encourage our system to generate more perturbed images that can lower the performances.

% \begin{table}[h]
% \setlength{\abovecaptionskip}{0.2cm}
% % \setlength{\belowcaptionskip}{-1cm}
% \centering 
% \renewcommand{\arraystretch}{1.2}
% \begin{tabular}{ccccccc}
% \hline
% \textbf{Data} & \textbf{Model} & \textbf{Standard Acc.} & \textbf{Counterfactual Acc.} & \textbf{Mixed Acc.} & CA/SA & MA/SA \\ \hline
% \multirow{2}{*}{Van.} & Van. & 95.38 & 55.12 & 67.49 & 0.58 & 0.71 \\
%  & Gray & 93.52 & 61.78 & 71.11 & 0.66 & 0.76 \\ \hline
% \multirow{2}{*}{Gray} & Van. & 95.38 & 64.36 & 73.62 & 0.67 & 0.77 \\
%  & Gray & 93.52 & 41.86 & 57.99 & 0.45 & 0.62 \\ \hline
% \end{tabular}
% \caption{Test performances of the model trained in a vanilla manner (denoted as Van.) or with grayscale model. We report two sets of performances, split by whether the counterfactual images are generated according to the vanilla model or the grayscale one (marked by the first column).}
% \label{tab:result:gray}
% \end{table}

\begin{table}[h]
\caption{Test performances of the model trained in a vanilla manner (denoted as Van.) or with grayscale model. We report two sets of performances, split by whether the perturbed images are generated according to the vanilla model or the grayscale one (marked by the first column).}
\centering 
\begin{tabular}{cccc}
\hline
\textbf{Data} & \textbf{Model} & \textbf{SA} & \textbf{FMR}  \\ \hline
\multirow{2}{*}{Van.} & Van. & 95.38 & 57.79 \\
 & Gray & 93.52  & 66.06 \\ \hline
\multirow{2}{*}{Gray} & Van. & 95.38 & 67.48  \\
 & Gray & 93.52 & 44.76 \\ \hline
\end{tabular}

\label{tab:result:gray}
\end{table}

In addition, we visualize some perturbed images generated according to each model and show them in Figure~\ref{fig:compare:gray}. We can see some evidence that the graycale model forces the generation system to focus more on the shape of the object and less of the color of the images.
We find it particularly interesting that our system sometimes generates 
different images differently for different models 
while the resulting images deceive the respective model to make the same prediction. 

\begin{figure}[h]
    \centering
        \begin{subfigure}{0.15\textwidth}
            \centering
            \includegraphics[width=\textwidth]{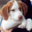} 
            \caption{dog} 
        \end{subfigure}
        \begin{subfigure}{0.15\textwidth}
            \centering
            \includegraphics[width=\textwidth]{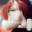} 
            \caption{horse} 
        \end{subfigure}
        \begin{subfigure}{0.15\textwidth}
            \centering
            \includegraphics[width=\textwidth]{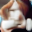} 
            \caption{horse} 
        \end{subfigure}
        \begin{subfigure}{0.15\textwidth}
            \centering
            \includegraphics[width=\textwidth]{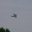} 
            \caption{plane} 
        \end{subfigure}
        \begin{subfigure}{0.15\textwidth}
            \centering
            \includegraphics[width=\textwidth]{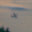} 
            \caption{ship} 
        \end{subfigure}
        \begin{subfigure}{0.15\textwidth}
            \centering
            \includegraphics[width=\textwidth]{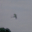} 
            \caption{bird} 
        \end{subfigure}
        \\
        \begin{subfigure}{0.15\textwidth}
            \centering
            \includegraphics[width=\textwidth]{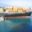} 
            \caption{ship} 
        \end{subfigure}
        \begin{subfigure}{0.15\textwidth}
            \centering
            \includegraphics[width=\textwidth]{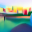} 
            \caption{car} 
        \end{subfigure}
        \begin{subfigure}{0.15\textwidth}
            \centering
            \includegraphics[width=\textwidth]{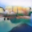} 
            \caption{plane} 
        \end{subfigure}
        \begin{subfigure}{0.15\textwidth}
            \centering
            \includegraphics[width=\textwidth]{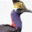} 
            \caption{bird} 
        \end{subfigure}
        \begin{subfigure}{0.15\textwidth}
            \centering
            \includegraphics[width=\textwidth]{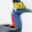} 
            \caption{ship} 
        \end{subfigure}
        \begin{subfigure}{0.15\textwidth}
            \centering
            \includegraphics[width=\textwidth]{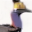} 
            \caption{ship} 
        \end{subfigure}
\caption{Visualization of the perturbed images generated by our system in evaluating the vanilla model (middle image of each group) and the grayscale model (third image of each group), with the original image shown. The caption for each image is either the original label or the predicted label by the corresponding model.}
\label{fig:compare:gray}
\end{figure}

\section{Sparse Submodel of VQGAN for Efficient Perturbation}
\label{sec:sparsevqgan}

% \subsubsection{Sparse Submodel of VQGAN for Efficient Perturbation}
% \hwc{I think this is an interesting technique, and I think we should emphasize this in the method part}

While our method will function properly as described above, 
we notice that the generation process still has a potential limitation: 
the bound-free perturbation of VQGAN will sometimes perturb the semantics 
of the images, generating results that will be rejected by the foundation model later and thus leading to a waste of computational efforts. 

To counter this challenge, we use a sparse variable selection method
to analyze the 
embedding dimensions of VQGAN to identify a subset of dimensions 
that is mainly responsible for the non-semantic variations. 

In particular, with a dataset $(\X, \Y)$ of $n$ samples, we first use VQGAN to generate a style-transferred dataset $(\X', \Y)$. During the generation process, we preserve the latent representations of input samples after the VQGAN encoder in the original dataset. We also preserve the final latent representations before the VQGAN decoder that are quantized after the iterations in the style-transferred dataset.
Then, we create a new dataset $(\E, \bL)$ of $2n$ samples, 
for each sample $(\e, l) \in (\E, \bL)$, 
$\e$ is the latent representation for the sample 
(from either the original dataset or the style-transferred one), 
and $l$ is labelled as 0 if the sample is from the original dataset and 1 if the style-transferred dataset. 

Then, we train $\ell_1$ regularized logistic regression model 
to classify the samples of $(\E, \bL)$. With $\w$ denoting the weights of the model, we solve the following problem
\begin{align*}
    \argmin_{\w} \sum_{(\e, l) \in (\E, \bL)} l(\e\w, l) + \lambda \Vert \w \Vert_1,
    % \vspace{-0.3cm}
\end{align*}
and the sparse pattern (zeros or not) of $\w$ will inform us about which dimensions are for the style. 

% To empirically study what is the most effective for discriminating the counterfactual images, we firstly decompose the the VQGAN generation process as outlined in Figure~\ref{fig:vqgan}:
% \begin{figure*}[t]
%     \centering
%     \includegraphics[width=1.0\textwidth]{imgs/vqgan_diagram.png}
%     \caption{Diagram showing how latent representations are extracted to form the binary classification dataset. The representations after VQGAN encoder are negative samples while the final latent representations before the VQGAN decoder after the iterations are positive samples.}
%     \hwc{thanks for the efforts in constructing this figure, maybe we should put figures 1 and 2 together, especially for the reason that VQGAN and quantization are not our contributions}
%     \label{fig:vqgan}
% \end{figure*}

We generate the flattened latent representations of input images after the VQGAN Encoder with negative labels. Following Algorithm~\ref{alg:main}, we generate the flattened final latent representations before the VQGAN decoder with positive labels. Altogether, we form a binary classification dataset where the number of positive and negative samples is balanced. The positive samples are the latent representations of perturbed images while the negative samples are the latent representations of input images. We set the split ratio of train and test set to be $0.8:0.2$. We perform the explorations on various datasets, i.e. MNIST, CIFAR-10, 9-class ImageNet and ImageNet. 

The classification model we consider is LASSO\footnote{\noindent  Although LASSO is originally a regression model, we probabilize the regression values to get the final classification results.} as it enables automatically feature selection with strong interpretability. We set the regularization strength to be $36.36$. We adopt saga~\citep{defazio2014saga} as the solver to use in the optimization process. The classification results are shown in Table~\ref{tab:lasso}.

% \begin{table}[t]
% \setlength{\abovecaptionskip}{0.2cm}
% \setlength{\belowcaptionskip}{-1cm}
% \small
% \centering 
% \renewcommand{\arraystretch}{1.3}
% \begin{tabular}{ccc}
% \hline
% \textbf{Data} & \textbf{Sparsity with L1 penalty} & \textbf{Test score with L1 penalty}  \\ \hline
% MNIST  & 97.99 & 78.50 \\
% CIFAR-10  & 98.45 & 78.00 \\ 
% 9-class ImageNet & 99.31 & 72.00 \\
% ImageNet  & 99.32 & 69.00  \\
%  \hline
% \end{tabular}
%      \caption{Classification results between vanilla and counterfactual images with LASSO.}
%      \hwc{I think this is too much space for a simple idea}
% \label{tab:lasso}
% \end{table}
We observe that the coefficient matrix of features can be far sparser than we expect. We take the result of 9-class ImageNet as an example. Surprisingly, we find that almost 99.31\% dimensions in average could be discarded when making judgements. We argue the preserved 0.69\% dimensions are highly correlated to VQGAN perturbation.
Therefore, we keep the corresponding 99.31\% dimensions unchanged and only let the rest 0.69\% dimensions participate in computation. Our computation loads could be significantly reduced while still maintain the competitive performance compared with the unmasked version\footnote{\noindent  We note that the overlapping degree of the preserved dimensions for each dataset is not high, which means that we need to specify these dimensions when facing new datasets.}. 

We conduct the run-time experiments on a single NVIDIA V100 GPU. Following our experiment setting, we evaluate a vanilla ResNet-18 on 9-class ImageNet and a vanilla ResNet-50 on ImageNet. As shown in Table~\ref{tab:efficiency}, the run-time on ImageNet can be reduced by 28.5\% with our sparse VQGAN. Compared with large-scale masked dimensions (\textit{i.e.,} 99.31\%), we attribute the relatively incremental run-time improvement (\textit{i.e.,} 12.7\% on 9-class ImageNet, 28.5\% on ImageNet) to the fact that we have to perform mask and unmask operations each time when calculating the model gradient, which offsets the calculation efficiency brought by the sparse VQGAN to a certain extent. 

\begin{table}
	\centering
	% \vspace{-0.3cm}
	\renewcommand{\arraystretch}{1.2}
	\captionof{table}{Classification results between vanilla and perturbed images with LASSO.}
	\begin{tabular}{c|c|c}
\hline
\textbf{Data} & \textbf{Sparsity} & \textbf{Test score}  \\ \hline
MNIST  & 97.99 & 78.50 \\
CIFAR-10  & 98.45 & 78.00 \\ 
9-class ImageNet & 99.31 & 72.00 \\
ImageNet  & 99.32 & 69.00  \\
 \hline
\end{tabular}
% \vspace{10pt}

\label{tab:lasso}
\end{table}

\begin{table}[h]
    \centering
    \renewcommand{\arraystretch}{1.2}
    \caption{Run-time Comparision between VQGAN and Sparse VQGAN.}
%         \begin{tabular}{cc}
% \hline
% \textbf{Method} & \textbf{Time} \\ \hline
% Unmasked  & 521.5\pm 1.2s \\
% Masked  & 455.4 \pm 1.1s \\ \hline
% \end{tabular}
    \begin{tabular}{ccc}
\hline
\multirow{2}{*}{\textbf{Method}} &
\multicolumn{2}{c}{\textbf{Time}} \\
  & \textbf{9-class ImageNet} & \textbf{ImageNet} \\ \hline
VQGAN   & $521.5\pm 1.2$s & $52602.4\pm 2.7$s\\
Sparse VQGAN   & $455.4\pm 1.2$s & $40946.1\pm 2.7$s\\ \hline
\textit{Improv.}  & 12.7\% & 28.5\%\\ \hline
\end{tabular}
    
    \label{tab:efficiency}
\end{table}

\section{Analysis of Samples that are Misclassified by the Model}
\label{sec:9classdetail}
We notice that,
the CLIP model has been influenced by the imbalance sample distributions across the Internet. 

In this experiment, we choose a larger step size so that the foundation model may not be able to maintain the image-label structure at the first perturbation. We report the Validation Rate (VR) which is the percentage of images validated by the foundation model that maintains the image-label structure. (In our official configurations, the step size value is small enough that the VR on each dataset is always 1. Therefore, we omit this value in the main experiments.) We present the results on 9-class ImageNet experiment to show the details for each category.

% We observe that the oracle model can tolerate a 
% much more significant perturbation over samples 
% labelled as \textit{Dog} (VR 0.95) or \textit{Cat} (VR 0.94)
% than samples labelled as \textit{Primate} (VR 0.48). The FMR value for \textit{Primate} images are much higher than other categories, creating an illusion that the evaluated models are robust against perturbed \textit{Primate} images. However, such an illusion is caused by the limitation that the oracle could only handle slightly perturbed samples. Therefore, the generated images will be biased towards the CLIP. 

% However, in our official configuration, we set a relatively smaller step size to perturb the image and obtain enough more perturbed images. Admittedly, the oracle's bias still exists, e.g., \textit{Primate} images are still easier than others. Nevertheless, as shown in Table~\ref{tab:result:case001}, considering the huge performance gap between oracle and the evaluated models, images that are easy for oracle are still hard for the evaluated models, which is enough for the applications in reality.

% In this section, we report the Validation Rate (VR) which is the percentage of images validated by the oracle that maintains the image-label structure. (In our official configurations, the step size value is small enough that the VR on each dataset is always 1. Therefore, we omit this value in the main experiments.) We present the results on 9-class ImageNet experiment to show the details for each category.

\begin{table*}[h]
\small 
\centering
\renewcommand{\arraystretch}{1.2}
\caption{Details of test on 9-class ImageNet for vanilla ResNet-18 (step size is 0.1, computation budget $\B$ is 50)}
\begin{tabular}{cccc}
\hline
\textbf{Type} & \textbf{SA}  & \textbf{VR}  & \textbf{FMR}    \\ \hline
Dog & 93.33  & 95.33  & 17.98  \\ 
Cat & 96.67  & 94.00  & 31.55  \\ 
Frog & 85.33  & 80.67  & 20.34  \\ 
Turtle & 84.67  & 78.67  & 29.03   \\ 
Bird & 91.33  & 96.00  & 28.13  \\ 
Primate & 96.00  & 48.00  & 62.21  \\ 
Fish & 94.00  & 76.67  & 45.33  \\ 
Crab & 96.00  & 87.33 & 19.87  \\ 
Insect & 93.33  & 78.00  & 33.88  \\ \hline
Total & 92.30  & 81.63  & 30.28  \\\hline
\end{tabular}

\label{tab:result:case}
\end{table*}

Table~\ref{tab:result:case} shows that the VR values for most categories are still higher than 80\%, some even reach 95\%, which means we produce sufficient number of perturbed images. However, we notice that the VR value for \textit{primate} images is quite lower compared with other categories, indicating around 52\% perturbed \textit{primate} images are blocked by the orcle. 

% We observe that the oracle model can tolerate much greater perturbations on images labeled as \textit{Dog} or \textit{Cat} (VR of 0.95 and 0.94 respectively) than on images labeled as \textit{Primate} (VR of 0.48). 
% While the FMR value for \textit{Primate} images are much higher than other categories,
% this apparent robustness to perturbed \textit{Primate} images is illusory, caused by the oracle's inability to handle heavily perturbed samples and resulting in a bias towards CLIP.

% We have discussed this category unbalance issue in Section~\ref{sec:discuss}.

As shown in Table~\ref{tab:result:case}, the FMR value for each category significantly drops compared with the SA value, indicating the weakness of trained models. An interesting finding is that the FMR value for \textit{Primate} images are quite higher than other categories, given the fact that more perturbed \textit{Primate} images are blocked by the foundation model. We attribute it to the limitation of foundation models. As the CLIP model has been influenced by the imbalance sample distributions across the Internet, it could only handle easy perturbed samples well. Therefore, the perturbed images preserved would be those that can be easily classified by the models.

\begin{table*}[h]
\small 
\centering
\renewcommand{\arraystretch}{1.2}
\caption{Details of test on 9-class ImageNet for vanilla ResNet-18 (step size is 0.001, computation budget $\B$ is 50)}
\begin{tabular}{cccc}
\hline
\textbf{Type} & \textbf{SA}  & \textbf{VR}  & \textbf{FMR}    \\ \hline
Dog & 93.33  & 100.00  & 18.09  \\ 
Cat & 96.67  & 100.00  & 28.60  \\ 
Frog & 85.33  & 100.00  & 20.72  \\ 
Turtle & 84.67  & 100.00  & 24.80   \\ 
Bird & 91.33  & 100.00  & 27.68  \\ 
Primate & 96.00  & 100.00  & 27.11  \\ 
Fish & 94.00  & 100.00  & 25.13  \\ 
Crab & 96.00  & 100.00 & 19.15  \\ 
Insect & 93.33  & 100.00  & 23.16  \\ \hline
Total & 92.30  & 100.00  & 23.94  \\\hline
\end{tabular}
\label{tab:result:case001}
\end{table*}

In our official configuration, we set a relatively smaller step size to perturb the image and obtain enough more perturbed images. As shown in Table~\ref{tab:result:case001}, using a smaller step size value and enough computation budget barely affect the overall results. In addition, with smller step size, we manage to perturb the image little by little and can get enough more perturbed images (\textbf{VR becomes 100 on every category}, indicating that all the images are perturbed and maintained their image-label structure). Admittedly, the foundation model's bias still exists here, e.g., the \textit{Primate} images (FMR = 28.11) are still easier than \textit{Dog} images (FMR = 18.09). However, considering the huge performance gap between the foundation model and the evaluated models, images that are easy for the foundation model are hard enough for the evaluated models (The FMR of \textit{Dogs} and \textit{Primate} images are closer and smaller compared with those in Table~\ref{tab:result:case}), which is sufficiently efficacious for real-world applications. Additionally, the employment of an ensemble of multiple foundation models in our methodology serves to provide a further layer of alleviation for the aforementioned issue.

\section{Discussions on the Societal Bias of Relying on Large Models}
\label{sec:societal_bias}
\subsection{Potential Negative Impacts of Foundation Models}
Although the bias incurred by foundation models is less detrimental than the biases arisen from fixed benchmark datasets, a more detailed discussion on the potential negative impacts is necessary. One potential bias of making vision models behave more like the foundation models is that the vision model may inherit the limitations and assumptions of foundation models’ training data and objective function. For example, foundation models’ training data may not cover all possible visual concepts or scenarios that are relevant to a given task; foundation models’ objective function may not align with the desired outcome or evaluation metric of a given task; foundation models’ natural language supervision may introduce ambiguities or inconsistencies that affect the model’s performance or interpretation. These limitations and assumptions may affect the generalization and robustness of vision models that rely on foundation models. Moreover, we add recent works that especially investigate the bias of foundation models, and guide the readers to it for further warning, e.g.,~\citep{menon2022task} and~\citep{zhou2022vlstereoset}.

\subsection{Societal Bias of Relying on Large Models}
Moreover, our method relies on large models, where their societal bias is still unclear, therefore a related discussion would be beneficial.

Large-scale models could leverage the rich knowledge and generalization ability encoded in the training stage. However, one potential societal bias of relying on large models’ supervision on preserving the perturbed image could be that it would privilege certain groups or perspectives over others based on social or cultural norms. As the data used to train the pre-trained models may be imbalanced, incomplete, or inaccurate, leading to biased representations of certain groups or concepts, the perturbed images preserved by the pre-trained models may reflect stereotypes, or discrimination against certain groups of people based on their race, gender, age, religion, etc., which may be harmful, offensive, or deceptive to the users. Bridging the gap between the pre-trained model and the evaluated vision models will make the vision models inherit the limitations of pre-trained models, which have adverse consequences for people who are affected by them, such as reinforcing stereotypes, discrimination, or exclusion.

We add recent works that investigate the societal bias of large models, and guide the readers to it for further warning, \textit{e.g.,}~\citep{wang2022fairclip}.

% \section{Experiments on Specialized Application Scenarios}
% \label{zero-shot-dg}
% We conduct the following experiment to explore the applications on medical image datasets. 

\section{Experiments on the Zero-shot Adversarial Robustness of CLIP}
\label{zero-shot-adv}
We conduct the following experiment to compare the adversarial vulnerability between CLIP and robust ViT-like model pre-trained checkpoints of XCiT-L12~\citep{debenedetti2022light} from the RobustBench Leaderboard~\citep{croce2020robustbench}. The results are shown in Table~\ref{tab:zs-adv}. We find that the vanilla CLIP shows a better robustness performance under our quick experiments through FGSM attack. However, if we continue the attack process, we will eventually obtain the adversary that changes the CLIP's classification decision to the targeted class. 

\begin{table}[h]
    \centering
    \renewcommand{\arraystretch}{1.2}
    \caption{Comparison of the zero-shot adversarial robustness of CLIP with pretrained robust model. We find that CLIP shows a better robustness performance compared with XCiT-L12. We note that the CLIP's classification decision can be changed to the targeted class as attack continues.} 
    \begin{tabular}{ccccccc}
\hline
\multirow{2}{*}{\textbf{Step}} &
\multicolumn{2}{c}{\textbf{Target loss}} &
\multicolumn{2}{c}{\textbf{p[true=0]}} &
\multicolumn{2}{c}{\textbf{p[target=1]}}\\
  & \textbf{CLIP} & \textbf{ XCiT-L12} & \textbf{CLIP} & \textbf{ XCiT-L12} & \textbf{CLIP} & \textbf{ XCiT-L12}\\ \hline
0 & 8.621 & 4.712 & 0.6749 & 0.7437 & 0.0052 & 0.0728 \\ 
20 & 2.715 & 1.605 & 0.5083 & 0.4074 & 0.0986 & 0.2009 \\
40 & 2.316 & 0.8877 & 0.4007 & 0.2562 & 0.1357 & 0.3116 \\
60 & 1.684 & 0.7420  & 0.2177 & 0.1407 & 0.2144 & 0.4760 \\
80 & 1.540 & 0.6520 & 0.1813 & 0.1338 & 0.3335 & 0.5210\\ \hline
% Sparse VQGAN   & $455.4\pm 1.2$s & $40946.1\pm 2.7$s\\ \hline
% \textit{Improv.}  & 12.7\% & 28.5\%\\ \hline
\end{tabular}
    \label{tab:zs-adv}
\end{table}

% 0	8.621         4.712    	    0.6749         0.7437    	    0.0052         0.0728
% 20	2.715         1.605    	    0.5083         0.4074    	    0.0986         0.2009
% 40	2.316         0.8877    	    0.4007         0.2562    	    0.1357         0.3116
% 60	1.684         0.7420    	    0.2177         0.1407    	    0.2144         0.4760
% 80	1.540         0.6520    	    0.1813         0.1338    	    0.3335         0.5210

Fortunately, in production, one can use simpler techniques such as gradient masking to protect CLIP's weights from malicious users, thus, the opportunities of the CLIP being attacked from a white-box manner are quite low. In terms of black-box attacks, CLIP actually shows a strong resilience toward the adversarial samples generated for other models, for which we also have some supporting evidence: In Appendix~\ref{sec:fgsm}, we generate the images with the FGSM attack by the tested model. Table~\ref{tab:results:fgsm} shows that initiating with the FGSM adversarial examples barely affects the FMR, which implies that CLIP succeeds in defending these black-box adversarial images and preserving the hard ones such that the FMR does not change significantly (Otherwise, CLIP will discard heavily perturbed images and preserve easy ones with minor perturbation, leading to high FMR values). Furthermore, our approach incorporates an ensemble of foundation models, including robust models such as ConvNext-T-CvSt from the RobustBench Leaderboard, and employs a majority vote mechanism to validate the fidelity of the image-label relationships.

Thus, CLIP, especially when equipped with techniques to protect its weights and gradients, and coupled with an ensemble of robust foundation models, might be the closest one to serve as the ideal foundation models to maintain the image-label structure at this moment.

\section{List of Evaluated Models}
\label{sec:modellist}
The following lists contains all models we evaluated on various datasets with references and links to the corresponding source code.

\subsection{Pretrained VQGAN Model}
We use the checkpoint of vqgan\_imagenet\_f16\_16384 from \url{https://heibox.uni-heidelberg.de/d/a7530b09fed84f80a887/}
% \begin{itemize}
%     \item VQGAN: vqgan\_imagenet\_f16\_16384
% \end{itemize}

\subsection{Pretrained Foundation Models}
\begin{enumerate}
    \item Model weights of ViT-B/32 and usage code are taken from \url{https://github.com/openai/CLIP}
    \item CoCa~\citep{Yu2022CoCaCC} \url{https://github.com/lucidrains/CoCa-pytorch}
    \item ConvNeXt-T-CvSt~\citep{singh2023revisiting} \url{https://github.com/nmndeep/revisiting-at}
\end{enumerate}

% \begin{itemize}
%     \item CLIP: ViT-B/32
% \end{itemize}

\subsection{Timm Models Trained on ImageNet~\citep{rw2019timm}}
Weights are taken from \url{https://github.com/rwightman/pytorch-image-models/tree/master/timm/models}

\begin{enumerate}
    \item ResNet50~\citep{he2016deep}
    \item ViT~\citep{dosovitskiy2020image}
    \item DeiT~\citep{touvron2021training}
    \item Twins~\citep{chu2021twins}
    \item Visformer~\citep{chen2021visformer}
    \item Swin~\citep{liu2021swin}
    \item ConvNeXt~\citep{liu2022convnet}
\end{enumerate}

\subsection{Robust ResNet50 Models}
\begin{enumerate}
    \item ResNet50 SIN+IN~\citep{geirhos2018imagenettrained} \url{https://github.com/rgeirhos/texture-vs-shape}
    \item ResNet50 ANT~\citep{rusak2020increasing} \url{https://github.com/bethgelab/game-of-noise}
    \item ResNet50 ANT+SIN~\citep{rusak2020increasing} \url{https://github.com/bethgelab/game-of-noise}
    \item ResNet50 Augmix~\citep{hendrycks2019augmix} \url{https://github.com/google-research/augmix}
    \item ResNet50 DeepAugment~\citep{hendrycks2021many} \url{https://github.com/hendrycks/imagenet-r}
    \item ResNet50 DeepAugment+Augmix~\citep{hendrycks2021many} \url{https://github.com/hendrycks/imagenet-r}
    \item ResNet50 Discrete Adversarial Training (DAT)~\citep{mao2022enhance} \url{https://github.com/alibaba/easyrobust}
\end{enumerate}

\subsection{Additional Image Generators}
\begin{enumerate}
    \item Efficient-VDVAE~\citep{hazami2022efficient} \url{https://github.com/Rayhane-mamah/Efficient-VDVAE}
    \item Improved DDPM~\citep{nichol2021improved} \url{https://github.com/open-mmlab/mmgeneration/tree/master/configs/improved_ddpm}
    \item ADM~\citep{dhariwal2021diffusion} \url{https://github.com/openai/guided-diffusion}
    \item StyleGAN~\citep{sauer2022stylegan} \url{https://github.com/autonomousvision/stylegan_xl}
\end{enumerate}

\subsection{Pretrained XCiT-L12 Model}
Model weights of XCiT-L12~\citep{debenedetti2022light} are taken from \url{https://github.com/dedeswim/vits-robustness-torch}

\section{Leaderboards for Robust Image Model}
We launch leaderboards for robust image models. The goal of these leaderboards are as follows:
\begin{itemize}[leftmargin=*]
    \item To keep on track of state-of-the-art on each adversarial vision task and new model architectures with our dynamic evaluation process.
    \item To see the comparison of robust vision models at a glance (\textit{e.g.,} performance, speed, size, \textit{etc.}). 
    \item To access their research papers and implementations on different frameworks. 
\end{itemize}

We offer a sample of the robust ImageNet classification leaderboard in supplementary materials. 

\section{Additional Perturbed Image Samples}
\label{sec:moresample}

\begin{figure}[h]
    \centering
    \captionsetup[subfigure]{labelformat=empty}
    % \captionsetup{font={small}}
        \begin{subfigure}{0.13\textwidth}
            \centering
            \includegraphics[width=\textwidth]{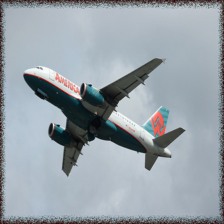} 
            \caption{airliner} 
            
        \end{subfigure}
        \begin{subfigure}{0.13\textwidth}
            \centering
            \includegraphics[width=\textwidth]{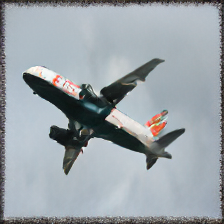}
            \caption{wing} 
            \label{fig:1sin}
            
        \end{subfigure}
        \begin{subfigure}{0.13\textwidth}
            \centering
            \includegraphics[width=\textwidth]{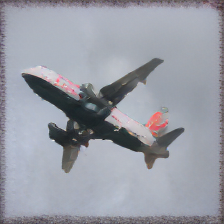}
            \caption{shuttle} 
        \end{subfigure}
        \begin{subfigure}{0.13\textwidth}
            \centering
            \includegraphics[width=\textwidth]{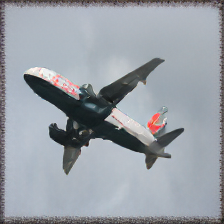}
            \caption{warplane} 
            \label{fig:1antsin}
        \end{subfigure}
        \begin{subfigure}{0.13\textwidth}
            \centering
            \includegraphics[width=\textwidth]{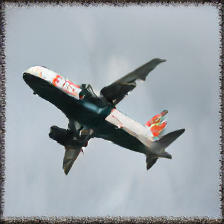}
            \caption{wing} 
        \end{subfigure}
        \begin{subfigure}{0.13\textwidth}
            \centering
            \includegraphics[width=\textwidth]{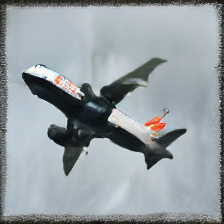}
            \caption{warplane} 
            \label{fig:1deepaug}
        \end{subfigure}
        \begin{subfigure}{0.13\textwidth}
            \centering
            \includegraphics[width=\textwidth]{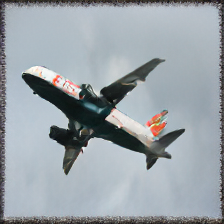}
            \caption{warplane}
            \label{fig:1deepaug+am}
        \end{subfigure}
        \\
        \begin{subfigure}{0.13\textwidth}
            \centering
            \includegraphics[width=\textwidth]{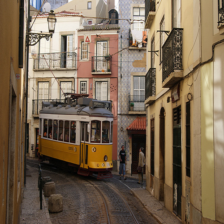}
            \caption{tram} 
        \end{subfigure}
        \begin{subfigure}{0.13\textwidth}
            \centering
            \includegraphics[width=\textwidth]{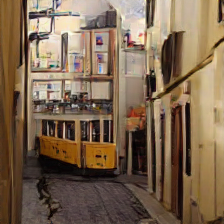}
            \caption{prison} 
            \label{fig:sin2}
        \end{subfigure}
        \begin{subfigure}{0.13\textwidth}
            \centering
            \includegraphics[width=\textwidth]{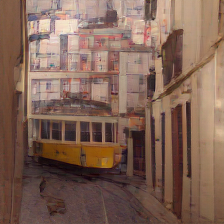}
            \caption{crib} 
        \end{subfigure}
        \begin{subfigure}{0.13\textwidth}
            \centering
            \includegraphics[width=\textwidth]{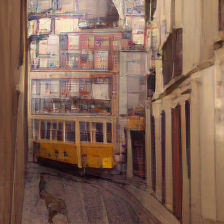} 
            \caption{vending} 
            \label{fig:antsin2}
        \end{subfigure}
        \begin{subfigure}{0.13\textwidth}
            \centering
            \includegraphics[width=\textwidth]{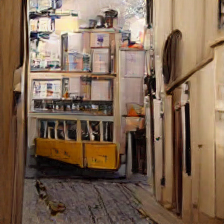} 
            \caption{chest} 
        \end{subfigure}
        \begin{subfigure}{0.13\textwidth}
            \centering
            \includegraphics[width=\textwidth]{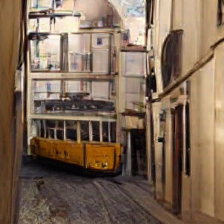}
            \caption{school bus} 
        \end{subfigure}
        \begin{subfigure}{0.13\textwidth}
            \centering
            \includegraphics[width=\textwidth]{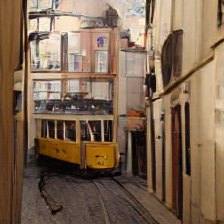}
            \caption{tram} 
        \end{subfigure}
        \\
        \begin{subfigure}{0.13\textwidth}
            \centering
            \includegraphics[width=\textwidth]{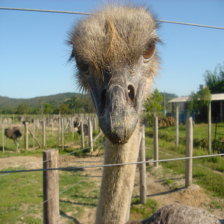} 
            \caption{ostrich} 
            
        \end{subfigure}
        \begin{subfigure}{0.13\textwidth}
            \centering
            \includegraphics[width=\textwidth]{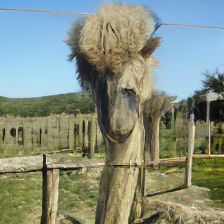} 
            \caption{dromedary} 
            \label{fig:3sin}
            
        \end{subfigure}
        \begin{subfigure}{0.13\textwidth}
            \centering
            \includegraphics[width=\textwidth]{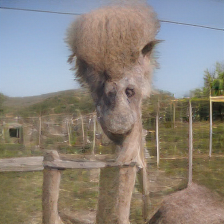} 
            \caption{llama} 
        \end{subfigure}
        \begin{subfigure}{0.13\textwidth}
            \centering
            \includegraphics[width=\textwidth]{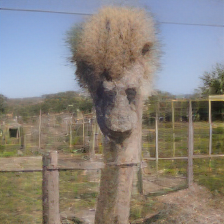}
            \caption{poodle} 
            \label{fig:3antsin}
        \end{subfigure}
        \begin{subfigure}{0.13\textwidth}
            \centering
            \includegraphics[width=\textwidth]{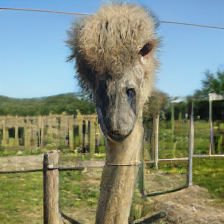}
            \caption{llama} 
        \end{subfigure}
        \begin{subfigure}{0.13\textwidth}
            \centering
            \includegraphics[width=\textwidth]{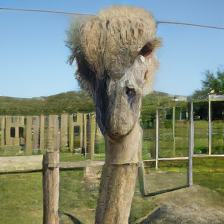}
            \caption{dromedary} 
            \label{fig:3deepaug}
        \end{subfigure}
        \begin{subfigure}{0.13\textwidth}
            \centering
            \includegraphics[width=\textwidth]{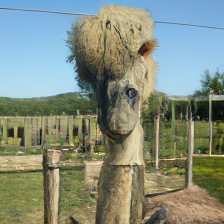}
            \caption{llama}
            \label{fig:3deepaug+am}
        \end{subfigure}
        \\
        \begin{subfigure}{0.13\textwidth}
            \centering
            \includegraphics[width=\textwidth]{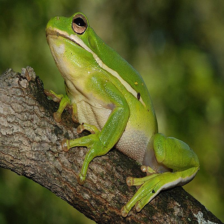} 
            \caption{tree frog} 
            
        \end{subfigure}
        \begin{subfigure}{0.13\textwidth}
            \centering
            \includegraphics[width=\textwidth]{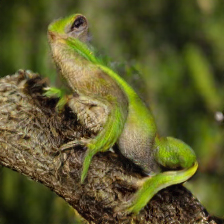}
            \caption{mamba} 
            \label{fig:4sin}
            
        \end{subfigure}
        \begin{subfigure}{0.13\textwidth}
            \centering
            \includegraphics[width=\textwidth]{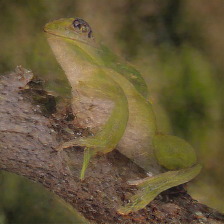}
            \caption{greenlizard} 
        \end{subfigure}
        \begin{subfigure}{0.13\textwidth}
            \centering
            \includegraphics[width=\textwidth]{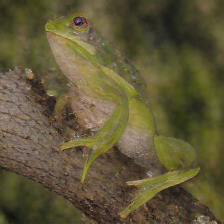}
            \caption{greenlizard} 
            \label{fig:4antsin}
        \end{subfigure}
        \begin{subfigure}{0.13\textwidth}
            \centering
            \includegraphics[width=\textwidth]{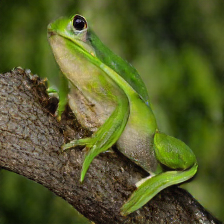}
            \caption{greenlizard} 
        \end{subfigure}
        \begin{subfigure}{0.13\textwidth}
            \centering
            \includegraphics[width=\textwidth]{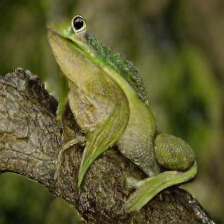}
            \caption{chameleon} 
            \label{fig:4deepaug}
        \end{subfigure}
        \begin{subfigure}{0.13\textwidth}
            \centering
            \includegraphics[width=\textwidth]{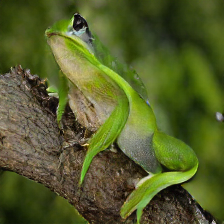}
            \caption{anole}
            \label{fig:4deepaug+am}
        \end{subfigure}
        \\
        \begin{subfigure}{0.13\textwidth}
            \centering
            \includegraphics[width=\textwidth]{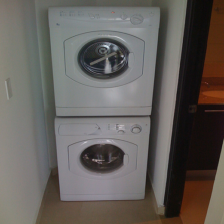}
            \caption{washer} 
            
        \end{subfigure}
        \begin{subfigure}{0.13\textwidth}
            \centering
            \includegraphics[width=\textwidth]{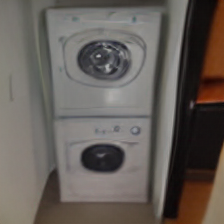}
            \caption{speaker} 
            \label{fig:5sin}
            
        \end{subfigure}
        \begin{subfigure}{0.13\textwidth}
            \centering
            \includegraphics[width=\textwidth]{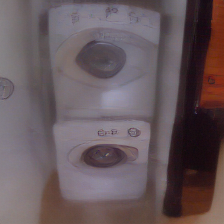}
            \caption{iPod} 
        \end{subfigure}
        \begin{subfigure}{0.13\textwidth}
            \centering
            \includegraphics[width=\textwidth]{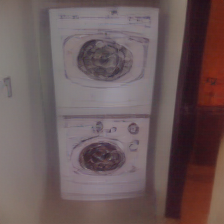}
            \caption{microwave} 
            \label{fig:5antsin}
        \end{subfigure}
        \begin{subfigure}{0.13\textwidth}
            \centering
            \includegraphics[width=\textwidth]{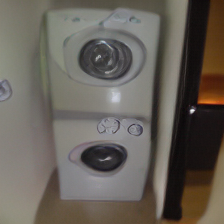}
            \caption{remote} 
        \end{subfigure}
        \begin{subfigure}{0.13\textwidth}
            \centering
            \includegraphics[width=\textwidth]{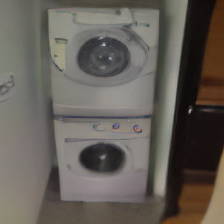}
            \caption{iPod} 
            \label{fig:5deepaug}
        \end{subfigure}
        \begin{subfigure}{0.13\textwidth}
            \centering
            \includegraphics[width=\textwidth]{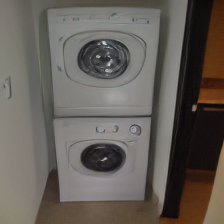}
            \caption{speaker}
            \label{fig:5deepaug+am}
        \end{subfigure}
\caption{Visualization of the images generated by our system in evaluating the common corruption robust model, with the original image shown (left image of each row). The caption for each image is either the original label or the predicted label by the corresponding model. The evaluated models are SIN, ANT, ANT+SIN, Augmix, DeepAug and DeepAug+AM from left to right.}
\label{fig:style-robust-all}
\end{figure}

In Figure~\ref{fig:style-robust-all}, we provide additional perturbed images generated according to each model. We have similar observations to Section~\ref{sec:robust}. First, the generated perturbed images exhibit diversity that many other superficial factors of the data would be covered, \textit{i.e.,} texture, shape and styles. Second, our method could recognize the model properties, and automatically generate those hard perturbed images to complete the evaluation.

In addition, the generated images show a reasonable transferability in Table~\ref{tab:result:transfer}, indicating tha our method can be potentially used in a broader scope: 
we can also leverage the method to generate a static set of images and set a benchmark dataset to help the development of robustness methods. Therefore, we also offer two static benchmarks in supplementary materials that are generated based on CNN architecture, \textit{i.e.,} ConvNext and transformer variant, \textit{i.e.,} ViT, respectively.

\section{Discussion on the realism of the generated images}
\label{sec:realism}
We notice that some generated images look unnatural, as the generated images being realistic is not part of the optimization function. We acknowledge that making the generated images appear more natural will be a further desideratum, as this contributes to enhancing the human-perceptible interpretability. 

Nonetheless, the current research agenda of the robustness evaluation community is still to encourage the evaluation to expose the model's weakness, such as to expose and eliminate the model's learning of spurious correlation in rare cases.

Similar evidence can be found in~\citep{xiao2023masked}, where the authors utilize masked images as counterfactual samples for robust fine-tuning. In this paper, the authors argue that masked images can break the spurious correlation between features and labels that may degrade OOD robustness, and that feature-based distillation with the pre-trained model on these counterfactual samples can achieve a better trade-off between IID and OOD performance. According to our second desideratum, our generated counterfactual images might also look unnatural. However, although it appears unnatural, it is beneficial in uncovering and eliminating spurious correlations for enhancing the model robustness.

\end{document}